\documentclass[10pt]{article} 
\usepackage[preprint]{tmlr}

\usepackage{amsmath,amsfonts,bm}

\def\eqref#1{equation~\ref{#1}}

\def\1{\bm{1}}

\def\rv{{\textnormal{v}}}

\DeclareMathAlphabet{\mathsfit}{\encodingdefault}{\sfdefault}{m}{sl}
\SetMathAlphabet{\mathsfit}{bold}{\encodingdefault}{\sfdefault}{bx}{n}

\newcommand{\E}{\mathbb{E}}

\newcommand{\R}{\mathbb{R}}

\usepackage{xcolor}
\usepackage{hyperref}
\hypersetup{colorlinks=true,
linkcolor={blue},citecolor={blue},
urlcolor={blue}}
\usepackage{xurl}
\usepackage{soul}
\usepackage{graphicx, array, blindtext}
\graphicspath{{./figures/}}
\usepackage[linesnumbered]{algorithm2e}
\usepackage{todonotes}
\usepackage{url}
\usepackage{booktabs}
\usepackage{amsfonts}
\usepackage{caption}
\usepackage{subcaption}
\usepackage{mathbbol}
\usepackage{multirow}
\usepackage{amsmath}
\usepackage{amsthm}
\usepackage{mathabx}
\usepackage{sidecap}
\usepackage{sidecap}
\usepackage{pbox}
\usepackage{arydshln}
\usepackage{hyperref}
\usepackage{cleveref}
\usepackage{authblk}
\usepackage{float}

\title{Why Online Reinforcement Learning is Causal}

\author{\name Oliver Schulte \email oschulte@cs.sfu.ca \\
      \addr School of Computing Science, Simon Fraser University, Vancouver, Canada
      \AND
      \name Pascal Poupart \email ppoupart@uwaterloo.ca \\
      \addr Cheriton School of Computer Science, University of Waterloo, Waterloo, Canada
}

\def\month{MM}  
\def\year{YYYY} 
\def\openreview{\url{https://openreview.net/forum?id=XXXX}} 

\newtheorem{definition}{Definition}

\newtheorem{observation}{Observation}
\newtheorem{proposition}{Proposition}
\newtheorem{lemma}{Lemma}

\def\defterm#1{\textbf{#1}}
\def\bs#1{\boldsymbol{#1}}
\def\set#1{\bs{#1}}

\newcommand{\X}{X}
\newcommand{\x}{x}
\newcommand{\Y}{Y}
\newcommand{\y}{y}

\renewcommand{\rv}{V}  
\newcommand{\rvalue}{v}
\newcommand{\xv}{X} 
\newcommand{\xvalue}{x}
\newcommand{\yv}{Y}  
\newcommand{\yvalue}{y}
\newcommand{\ov}{O}  
\newcommand{\ovalue}{o}
\newcommand{\lv}{Z}  
\newcommand{\lvalue}{z}
\newcommand{\uv}{U}  
\newcommand{\uvalue}{u}
\newcommand{\sv}{S}  
\newcommand{\svalue}{s}
\newcommand{\wv}{W}  
\newcommand{\wvalue}{w}
\newcommand{\ns}[1]{\underline{#1}}

\newcommand{\scm}{\mathcal{S}} 
\newcommand{\DIM}{\mathcal{D}}
\newcommand{\prior}{b}
\newcommand{\policy}{\pi}
\newcommand{\marginal}{\mu}
\renewcommand{\b}{b}

\renewcommand{\S}{S} 
\newcommand{\B}{B} 

\newcommand{\Z}{Z}
\newcommand{\z}{z}
\renewcommand{\E}{E} 
\newcommand{\A}{A} 
\renewcommand{\R}{R}
\renewcommand{\a}{a}

\renewcommand{\r}{r}
\newcommand{\Q}{Q}
\newcommand{\V}{V}
\newcommand{\ado}{\mathit{do}}

\newcommand{\G}{G}
\newcommand{\pa}{\mathit{pa}}
\newcommand{\PA}{\mathit{Pa}}

\newcommand{\cn}{\mathit{B}}
\newcommand{\cm}{\mathit{C}}
\newcommand{\ddn}{\mathit{D}}

\newcommand{\mprob}[1]{P^{#1}}

\newcommand{\SH}{\mathit{SH}}
\newcommand{\PH}{\mathit{PH}}
\newcommand{\CG}{\mathit{CG}}
\newcommand{\GH}{\mathit{GH}}
\newcommand{\SC}{\mathit{SC}}

\begin{document}

\maketitle

\begin{abstract}
    Reinforcement learning (RL) and causal modelling 
    naturally complement each other. The goal of causal modelling is to predict the effects of interventions in an environment, while the goal of reinforcement learning is to select interventions that maximize the rewards the agent receives from the environment. Reinforcement learning includes the two most powerful sources of information for estimating causal relationships: temporal ordering and the ability to act on an environment.  This paper examines which reinforcement learning settings  we can expect to benefit from causal modelling, and how. In online learning, the agent has the ability to interact directly with their environment, and learn from exploring it. Our main argument is that in online learning, conditional probabilities are causal, and therefore offline RL is the setting where causal learning has the most potential to make a difference. Essentially, the reason is that when an agent learns from their {\em own} experience, there are no unobserved confounders that influence both the agent's own exploratory actions and the rewards they receive. Our paper formalizes this argument.
    For offline RL, where an agent may and typically does learn from the experience of {\em others}, we describe previous and new methods for leveraging a causal model.
\end{abstract}

\section{Introduction: Causal Probabilities in Reinforcement Learning} \label{sec:intro}


The goal of decision-making in a Markov Decision Process (MDP) is to intervene in the environment to maximize the agent's cumulative reward. A key insight of causal decision theory is that the impact of an action should be estimated as a {\em causal effect}, not a correlation. 
Visits to the doctor correlate with illnesses, but avoiding seeing a doctor does not make a patient healthier~\cite[Ch.4.1.1]{Pearl2000}. Several causality researchers have therefore argued that reinforcement learning can benefit from adopting causal models to predict the effect of actions. This article is directed towards reinforcement learning researchers who want to explore the use of causal models. We provide  conceptual and theoretical foundations to facilitate the adoption of causal models by reinforcement learning researchers. We use as much as possible terminology, notation, and examples from reinforcement learning. A running example gives explicit computations that illustrate causal concepts. 
This paper can therefore serve as a short tutorial on causal modelling for RL researchers. An excellent long tutorial is provided by~\citet{bib:barenboim-tutorial}, and a recent survey by~\cite{deng2023causal}.  

The main question we address is {\em under what conditions causal modelling} provides a new approach to reinforcement learning. Our short answer is that {\em online learning}, where an agent learns a policy through interacting with the environment directly, is inherently causal: conditional probabilities estimated from online data are also causal probabilities (i.e., they represent the causal effect of interventions). In offline learning, where an agent may learn from a dataset collected through the experience of others, causal probabilities provide an alternative to conditional probabilities. 

\cite{Levine2020} assert that ``offline reinforcement learning is about making and answering counterfactual queries.'' Recent work on causal reinforcement learning has suggested utilizing the ability of causal models to evaluate {\em counterfactual} probabilities~\cite{bib:barenboim-tutorial,deng2023causal}. Extending our analysis from the effects of interventions to counterfactuals, we distinguish between {\em what-if} queries and {\em hindsight} queries. A what-if query concerns the results of deviating from an action taken; an example from a sports domain would be ``What if I had taken a shot instead of making a pass?''. A hindsight query conditions on an observed outcome. An example of a hindsight query would be ``I failed to score. What if I had taken a shot instead of making a pass?''. Our analysis indicates that in online RL, what-if counterfactuals can be evaluated using conditional probabilities, whereas hindsight counterfactuals require a causal model beyond conditional probabilities, even in online RL. 

We next give an outline of our analysis; formal details appear in the text below.



\paragraph{Overview.} 


Conditional probabilities measure the strength of associations or correlations, but not necessarily the causal effect of an action. 
Using Pearl's do operator, the \defterm{causal effect} of setting variable $\A$ to the value $\a$ given evidence covariates $\set{\X}$ can be written as a conditional probability of the form $P(\Y|\ado(\A = \a),\set{\X = \x}))$.
(The formal semantics for the $\ado$ operator is defined in ~\Cref{sec:background} below.) In the medical visit example, the strong correlation means that $P(\mathit{Illness}|\mathit{Visit})$ is high. However, making a person visit the doctor has no causal effect on their illness, so we have $P(\mathit{Illness}|\mathit{Visit})>>P(\mathit{Illness}|\ado(\mathit{Visit}))= P(\mathit{Illness})$. For an example of a conditional probability relevant to RL, consider $P(R_{t+1}|A_t = \a_t,S_t)$, the conditional probability of receiving reward $R$ at time $t+1$ given action $A_t$ and state $S_t$ at time $t$. A key question in this paper is {\em under what assumptions the conditional reward probability equals the causal reward probability} $P(R_{t+1}|\ado(A_t = \a_t),S_t)$. The answer depends on temporal information and confounding.

\paragraph{Temporal Ordering.} Since RL data include time stamps, we can leverage the fundamental principle that {\em causes do not succeed effects temporally}. Since rewards and next states follow previous states and actions, they can only be effects, not causes of previous states and actions.~\Cref{fig:reward} illustrates using the influence diagram formalism, how the causal ordering follows the temporal ordering.

\paragraph{Confounding and Online Learning} It can be shown that a conditional probability $P(\Y_{t+1}|\X_{t},\E_{\leq t})$ that predicts future events from past events is causal {\em unless} there is a common cause $Z_t$ of $\Y_{t+1}$ and $\X_{t}$ that is not included in the conditioning evidence $\E$; see~\Cref{fig:confounded}. We refer to an unobserved common cause as a \defterm{confounder}. {\em In which RL settings can we expect rewards to be confounded with actions?} The answer depends on different cases, as summarized in~\Cref{fig:cases}. 
We adopt the fundamental RL distinctions between (1) online and offline learning~\citep{Levine2020}, (2) on-policy and off-policy evaluation, and (3) complete vs. partial observability.

\begin{figure}
     \centering
     \begin{subfigure}[b]{0.45\textwidth}
         \centering
         \includegraphics[width=\textwidth]{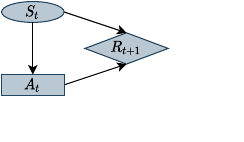}
         \caption{A completely observable reward model, with no latent variables.}
         \label{fig:temporal}
     \end{subfigure}
      \begin{subfigure}[b]{0.40\textwidth}
         \centering
         \includegraphics[width=\textwidth]{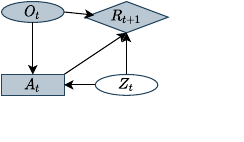}
         \caption{Partially observable reward model, with a latent confounder. }
         \label{fig:confounded}
     \end{subfigure}
     \caption{Dynamic influence diagrams for a generic reward model. We follow the conventions of influence diagrams to distinguish state variables, actions, and rewards. Observed variables are gray, latent variables white. 
     ~\Cref{fig:temporal}: States and actions temporally precede rewards. Therefore rewards to not cause states/actions, and reward probabilities are causal (i.e., $P(R_{t+1}|\ado(A_t = \a_t),S_t) = P(R_{t+1}|A_t = \a_t,S_t)$), {\em  unless} there is an unobserved confounder. ~\Cref{fig:confounded}: The environment state $S_t = (O_t,Z_t)$ comprises an observation signal $O_t$ and a latent part $Z_t$. The unobserved variable $Z_t$ is a latent common cause (confounder) of both actions and rewards. Because of the confounder, conditional probabilities generally do not correctly estimate the causal effects of actions (i.e., $P(R_{t+1}|\ado(A_t = \a_t),S_t) \neq P(R_{t+1}|A_t = \a_t,S_t)$).
     \label{fig:reward}}
\end{figure}



\begin{figure}[hbtp]
     \centering
     \begin{subfigure}[b]{0.45\textwidth}
         \centering
         \includegraphics[width=\textwidth]{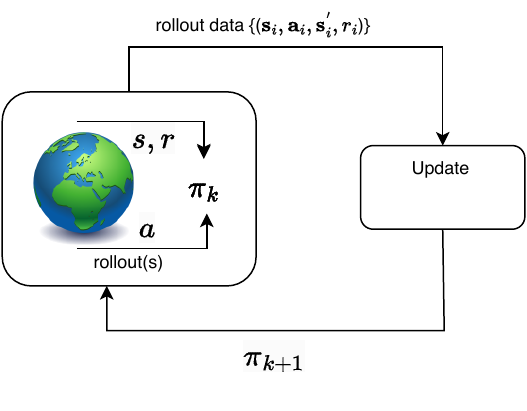}
         \caption{Online RL.}
         \label{fig:levine-online}
     \end{subfigure}
      \begin{subfigure}[b]{0.45\textwidth}
         \centering
         \includegraphics[width=\textwidth]{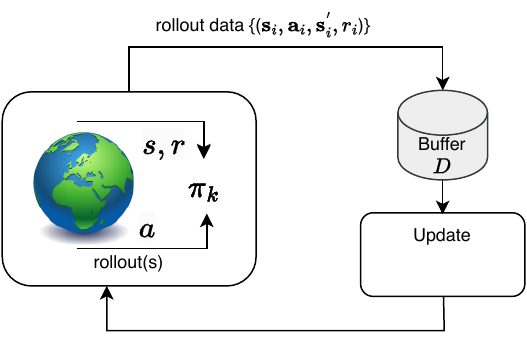}
         \caption{Off-policy RL.\label{fig:levine-offpolicy}}
     \end{subfigure}
     \caption{Online RL settings; Figures adapted from~\cite{Levine2020}. 
     ~\Cref{fig:levine-online}: In classic \textbf{online} RL, the policy $\policy_{k}$ is updated with streaming data collected by $\policy_{k}$ itself.
     ~\Cref{fig:levine-offpolicy}:  In classic \textbf{off-policy} RL, 
     the agent’s online experience is appended to a data buffer (also called a replay buffer) $D$, each new policy $\policy_{k}$ collects additional data, such that $D$ comprises samples from $\policy_{0},\ldots,\policy_{k}$, and all of this data is used to train an updated new policy $\policy_{k+1}$. Both online settings {\em satisfy observation-equivalence} because the policies used to generate the data are based on the same observations as the learned policy.
     \label{fig:online-picture}}
\end{figure}

\begin{figure}[hbtp]
     \centering
     \begin{subfigure}[b]{0.45\textwidth}
         \centering
         \includegraphics[width=\textwidth]{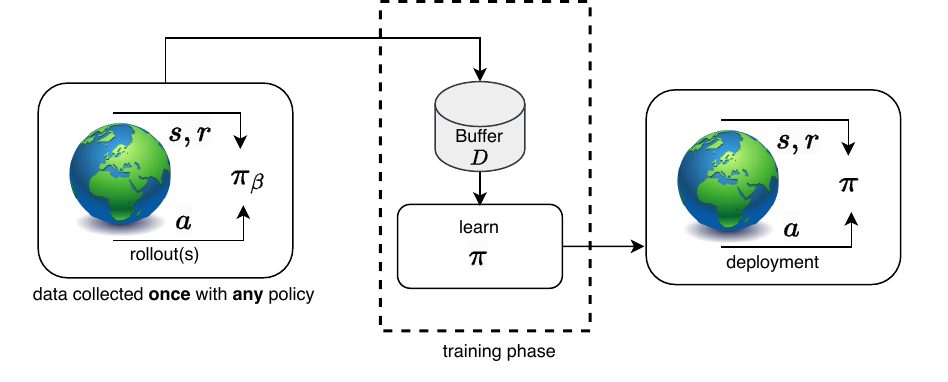}
         \caption{Offline RL with observation-equivalence.}
         \label{fig:levine-offline}
     \end{subfigure}
      \begin{subfigure}[b]{0.45\textwidth}
         \centering
         \includegraphics[width=\textwidth]{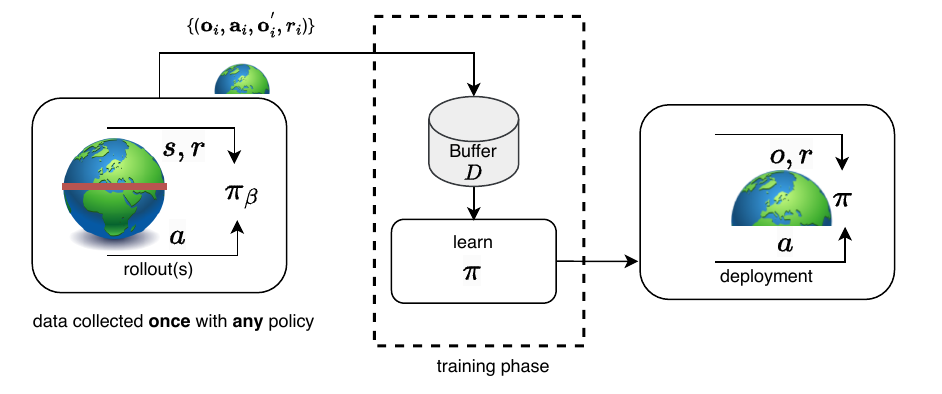}
         \caption{Offline RL without observation-equivalence.}
         \label{fig:partial-offline}
     \end{subfigure}
     \caption{\textbf{Offline} RL-learning employs a dataset $D$ collected by some (potentially unknown) behavior policy $\policy_{\beta}$. The dataset is collected once, and is not altered during training. The training process does not interact with the environment directly, and the policy is only deployed after being fully trained.~\Cref{fig:levine-offline}~\citep{Levine2020}: offline RL with observation-equivalence where the behavioral policy $\policy_{\beta}$ and the learned policy $\policy$ are based on the same observation signal.~\Cref{fig:partial-offline}: offline RL without observation-equivalence where the behavior policy $\policy_{\beta}$ has access to more observations than the learned policy $\policy$.   
     \label{fig:offline-picture}}
\end{figure}

\begin{figure}
     \centering
     \begin{subfigure}[b]{0.32\textwidth}
         \centering
         \includegraphics[width=\textwidth]{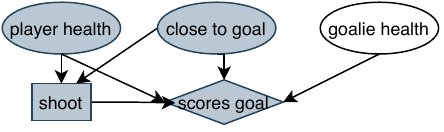}
         \caption{Online View.}
         \label{fig:first}
     \end{subfigure}
      \begin{subfigure}[b]{0.32\textwidth}
         \centering
         \includegraphics[width=\textwidth]{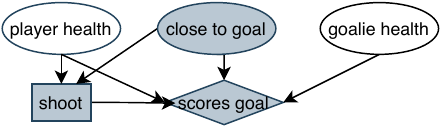}
         \caption{Offline View with confounder.\label{fig:confound}}
     \end{subfigure}
     \caption{Causal Graphs for a sports scenario like hockey or soccer. We follow the conventions of influence diagrams to distinguish state variables, actions, and rewards. All variables are binary. Variables observed by the learning agent are gray, latent variables white. Whether a player takes a shot depends on their location and whether they are injured. Likewise, the chance of their shot leading to a goal depends on their location and health. Thus Player Health is a common cause of the action and reward.
     ~\Cref{fig:first}: In the online setting, the athlete learns from their own experience, which includes their health.
     ~\Cref{fig:confound}: In the offline setting, the learner is different from the athlete, for example a coach, and does not observe the health of the behavioral agent. Player health is an unobserved confounder of action and reward.
     \label{fig:cbns}}
\end{figure}

Case 1: The learning agent can directly interact with its environment. For example a video game playing system can execute actions in the game and observe their effects~\citep{Mnih2015}. In this {\em online} setting 
illustrated in~\Cref{fig:online-picture}, the agent learns from their {\em own} experience exploring the environment. As an agent is transparent to itself, the agent is aware of the causes of its own actions. So {\em in an online setting there are no unobserved action causes}, and hence no confounders of exploratory actions and observed rewards. 

Case 2: The learning agent cannot directly interact with its environment. In this {\em offline} setting, the agent relies on the experience of another agent, such as an expert demonstrator. In this case, the learning agent may receive a different observation signal than the demonstrating/behavioral agent; see~\Cref{fig:partial-offline}.  
When the learning agent does not have access to all the causes of the decisions made by the behavioral agent,  {\em actions and rewards may be confounded} from the learner's perspective.~\cite{Zhang2020} give the example of a driving agent whose policy maps a state specifying other car velocities and locations  to the acceleration of its ego car. The driving agent is learning offline from driving demonstrations; the demonstrator's decisions are based also on the tail light of the car in front of it. From the perspective of the learning agent, the tail light confounds the state and the reward (no accident); see~\Cref{fig:driving} below.

~\Cref{fig:cbns} illustrates these distinctions in a simple sports setting (such as hockey or soccer). Player health is a common cause of the player's decision (e.g., to shoot) and the player's success (score a goal). In the first-person online setting, the athlete is aware of their own health. In the third-person offline setting, an observer such as coach, does not have access to the athlete's health.

An RL setting satisfies {\em observation-equivalence} if the behavioral and the learned policy are based on the same observation signal. The gist of our analysis is that {\em if the learning setting satisfies observation-equivalence, as it does in online RL, then causal effects and what-if counterfactuals can be estimated from conditional probabilities}. 

While online learning is sufficient for observation-equivalence, it is not necessary.  \cite{scholkopf2021towards} note that ``[Reinforcement learning] sometimes effectively directly estimates do-probabilities. E.g., on-policy learning estimates do-probabilities for the interventions specified by the policy''. In on-policy learning, the behavioral and the learned policy are the same, so they are observation-equivalent. Another sufficient condition is complete observability, where the environment is completely observable for both the behavioral and the learning agent.  
 For instance, the first phase of training the AlphaGo system was based on an offline dataset of games of Go masters~\citep{silver2016mastering}. Go is a completely observable board game with no hidden information. Under complete observability, the learning agent has access to the same observations as the behavioral agent, and therefore to the causes of the behavioral agent's decisions. Our overall conclusion is that {\em causal effects differ from conditional probabilities in the offline off-policy RL setting with partial observability}; see~\Cref{fig:cases}. We support this conclusion with theorems based on causality theory~\citep{Pearl2000}.

\begin{figure}
    \centering
    \includegraphics[width=0.7\textwidth]{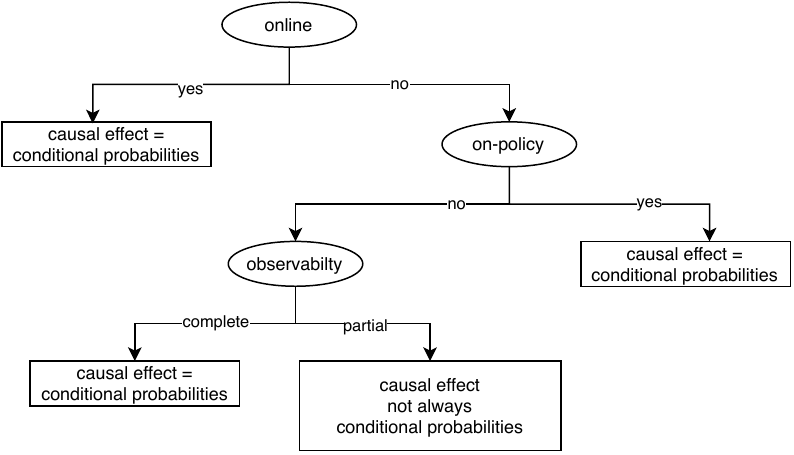}
    \caption{Reinforcement learning settings in which we can expect conditional probabilities to be equivalent to interventional probabilities.}
    \label{fig:cases}
\end{figure}

\paragraph{Paper Overview}


\begin{table}[ht]
\centering
\caption{A four-level causal hierarchy, which refines Pearl's three-level hierarchy Association-Intervention-Counterfactual~\citep{Pearl2000}. {\em Our analysis shows that in online RL, queries of the first three types can be computed from conditional probabilities.}}
\begin{tabular}{|l|l|p{3cm}|p{3cm}|}
\hline
Level          & Notation                   & Typical Question                                        & Example                                              \\ \hline
Association    & $P(R|\mathbf{S},A)$        & What reward follows after an agent chooses $A$? & How often does a shot lead to a goal?                \\ \hline
Intervention   & $P(R|\mathbf{S},\ado(A))$  & If I chose $A$, what will my reward be?                 & If I take a shot, will I score a goal?                \\ \hline
\begin{tabular}{l} What-if \\Counterfactual \end{tabular} & $P(R_{A}|\mathbf{S},B)$ & What if I had chosen $A$ instead of $B$?                & What if I had taken a shot instead of making a pass? \\ \hline \hline
\begin{tabular}{l} Hindsight \\Counterfactual \end{tabular} & $P(R_{A}|\mathbf{S},B,R')$ & What if I had chosen $A$ instead of $B$?                & I failed to score. What if I had taken a shot instead of making a pass? \\ \hline 
\end{tabular}

\label{tab:hierarchy}
\end{table}

The paper is organized following a ``ladder of causation'' as described by~\citep{Pearl2000}: A hierarchy of probabilistic statements that require causal reasoning of increasing complexity. The levels correspond to associations, interventions, and counterfactuals. ~\Cref{tab:hierarchy} illustrates these concepts in the RL setting. A {\em formal semantics} for each type of probability can be defined in terms of a generative model that is based on a causal graph. We analyze the relationship between online RL and each level of the causal hierarchy in a separate independent section: one section for online RL and interventional probabilities, one section for online RL and what-if counterfactuals, and one section for online RL and hindsight counterfactuals.


\textit{Paper Overview.} We review background on causal models, including Bayesian networks and structural causal models. 
Causal Bayesian networks are
a class of graphical models that provide an intuitive semantics for interventional probabilities. We review partially observable Markov Decision Processes (POMDPs) and show how a dynamic causal model can represent a POMDP. 
A causal graph that represents a POMDP
includes causes of both the agent's actions and the  environment's responses~\cite[Ch.17]{Russell2010}. A common approach to solving POMDPs involves transforming the POMPD to a belief MDP, where an agent's decisions are based on their current beliefs about the (partially) unobserved current state. We define a novel variant of a belief MDP that facilitates leveraging a causal model for reinforcement learning. 
Next we explain fundamental RL settings such as online, offline, on-policy, off-policy, and 
give an informal but rigorous argument for why we can expect online RL to satisfy observation-equivalence. 
Our main formal proposition states that given observation-equivalence and temporal ordering, causal effects and what-if counterfactuals coincide with conditional probabilities.
Therefore the reward model, transition model, and the expected value function capture causal effects and what-if counterfactuals when based on  conditional probabilities. In contrast, hindsight counterfactuals cannot be reduced to conditional probabilities, even in online RL. 

In offline RL without observation-equivalence, actions may be confounded with rewards, so causal effects may differ from conditional probabilities. We describe how a causal model can be used to compute interventional and counterfactual probabilities for offline reward/transition models and Q-functions. 
Our final section reviews related work on utilizing causal models for RL from the perspective of the online/offline distinction and describes several directions for future research. 



\section{Background: Causal Bayesian Networks} \label{sec:background}

\paragraph{Notation}~\Cref{tab:notation} summarizes the notation used in this paper, and previews the concepts introduced in the remainder.

\begin{table}[htb]
\begin{tabular}{|l|l|}
\hline
Notation                                                      & Meaning                                                                           \\ \hline
$\set{\rv}$                                                   & Random Variables. $\set{\rv} = \set{\ov} \cup \set{\lv}$                          \\ \hline
$\set{\xv}$                                                   & Generic set of variables. $\set{\xv} \subseteq \set{\rv}$.                        \\ \hline
$\set{\ov}$                                                   & Observed Variables.                                                               \\ \hline
$\set{\lv}$                                                   & Latent Variables                                                                  \\ \hline
$\set{\uv}$                                                   & Latent Source Variables. $\set{\uv} \subseteq \set{\lv}$                          \\ \hline
$\set{\ns{\rv}} = \set{\rv} - \set{\uv}$                          & Non-source variables with positive in degree                                      \\ \hline
$\G$                                                          & Directed Acyclic Graph (DAG); causal graph                                        \\ \hline
$\PA_{\rv}$ resp. $\PA_{i}$                                   & Parents/direct causes of variable $\rv$ resp. \$\textbackslash{}rv\_\{i\}         \\ \hline
$\cn$                                                         & Causal Bayesian Network                                                           \\ \hline
$\cm$                                                         & Probabilistic Structural Causal Model                                             \\ \hline
$f_{\rv}$ resp. $f_{i}$                                       & Deterministic local function for variable $\rv$ resp. \$\textbackslash{}rv\_\{i\} \\ \hline
$S$                                                           & State Space in POMDP                                                              \\ \hline
$\set{\S}$                                                    & State Variables in factored POMDP; $\set{\S} = \set{\ov} \cup \set{\lv} - \{\A\}$ \\ \hline
$\A$                                                          & Intervention Target; Decision Variable in Influence Diagram; Action in POMDP      \\ \hline
$\ado(\A=\hat{a})$                                             & Selecting action $\hat{a}$ as an intervention      
\\ \hline $\R$ & Reward 
\\ \hline
$\b(\set{\lv})$                                                     & Belief State                                                                      \\ \hline
$\policy$                                                     & Policy                                                                            \\ \hline
$\policy(\a|\langle \set{\ovalue},\b\rangle)$                 & Probability of action given current observation and belief state                  \\ \hline
$Q^{\policy}(\langle \set{\ovalue},\b\rangle,\a)$             & Expected Return given current observation, belief state, and action               \\ \hline
$Q^{\policy}(\langle \set{\ovalue},\b\rangle,\ado(\hat{\a}))$ & Expected Return given current observation, belief state, and action/intervention  \\ \hline
\end{tabular}
\caption{Notation used in this paper}
\label{tab:notation}
\end{table}

 In this section we define the semantics for the first two levels of causal hierarchy in~\Cref{tab:hierarchy}, observational associations and intervention probabilities, based on \textit{causal Bayesian networks} (CBNs). CBNs specify intervention probabilities through a truncation semantics~\citep{Pearl2000}[Ch.1.3]. Their parameters are
 conditional probability parameters of the form $P(\mathit{child\_value|parent\_values)}$. While causal Bayesian networks are easy to interpret, recent research has focused on  \textit{structural causal models}~\citep{Pearl2000}[Ch.2.2],\citep{scholkopf2021towards}, which combine latent variables with Bayesian networks. Latent variables enhance the expressive power of causal graphs to define a formal semantics for counterfactual probabilities, which we describe in \Cref{sec:counter} below. In the next~\Cref{sec:pomdp}, we show how POMDPs can be represented using causal Bayesian networks. 
 

\subsection{Causal Bayesian Networks} \label{sec:bns}

A \textbf{causal graph} is a directed acyclic graph (DAG) whose nodes are a set of random variables $\set{\rv} = \{\rv_1,\ldots,\rv_n\}$.  Throughout the paper we assume that random variables are discrete. The definitions can easily be extended to continuous random variables. A \textbf{causal Bayesian network}~\citep[Ch.1.3]{Pearl2000},\citep{bib:cooper-yoo}, or \textbf{causal network} for short, is a causal graph $\G$ parametrized by conditional probabilities $P(\rvalue_i|\pa_i)$ for each possible child value $\rv_i$ and joint parent state $\pa_i$. 

A causal network $\cn$ defines a \textbf{joint distribution} over random variables $\set{\rv}$ through the product formula

\begin{equation} \label{eq:bn-observe}
    \mprob{\cn}(\set{\rv} = \set{\rvalue}) = \prod_{i=1}^n P(\rvalue_i|\pa_i)
\end{equation}

where $x_i$ resp. $\pa_i$ are the values given to node $\rv_i$ resp. its parents $\PA_i$ by the assignment $\set{\rv} = \set{\rvalue}$. Here and below we often omit the model index when the model is fixed by context.

A causal network $\cn$ also defines a joint \textbf{interventional distribution} through the {\em truncation semantics} as follows. Write $\ado(A = \hat{a})$ 
to denote an intervention that sets variable $A$ to value $\hat{a}$. In the RL context, $\A$ represents an action; in this section, it represents a generic intervention target. 
 The effect of the intervention is to {\em change the causal network} $\cn$ to a \textbf{truncated network} $\cn_{\ado(A = \hat{a})}$, in which node $\A$ has no parents and with probability 1 takes on the value $\hat{a}$. Removing the parents of $\A$ represents that  the parents of $\A$ no longer influence its value after the intervention. 
 Given an intervention on variable $A$, the truncated network $\cn_{\ado(A = \hat{a})}$ defines a joint distribution as follows:

\begin{equation} \label{eq:truncate}
    \mprob{\cn}_{\ado(A = \hat{a})}(\set{\rv} = \set{\rvalue},\A = \a) = \prod_{\rv_i\neq A} P(\rvalue_i|\pa_i) \delta(a,\hat{a})
\end{equation}

where $x_i$ resp. $\pa_i$ are the values given to node $\rv_i$ resp. its parents $\PA_i$ by the assignment $(\set{\rv} = \set{\rvalue},A = a)$ and $\delta(a,\hat{a}) = 1$ if $a = \hat{a}$, 0 otherwise. For conditional probabilities that represent {\bf causal effects} ~\citet[Dfn.3.2.1.]{Pearl2000}, uses conditional notation such as

\begin{equation*}
\mprob{\cn}(\set{\yv}|\set{\xv}=\set{\xvalue},\ado(A = \hat{a})) \equiv \mprob{\cn}_{\ado(A = \hat{a})}(\set{\yv}|\set{\xv}=\set{\xvalue}) \mbox{ for disjoint } A,\set{\yv},\set{\xv}
\end{equation*}
to denote the causal effect on a list of outcome variables $\set{\yv}$ due to setting variable $\A$ to value $\hat{a}$ after observing evidence $\set{\xv}=\set{\xvalue}$. The truncation semantics easily extends to intervening on multiple variables by removing all their links.

 \paragraph{Remarks on Notation.}  The $\hat{\a}$ notation does not indicate an quantity estimated from data, but an intervention. We sometimes use the syntactic sugar $\hat{A}$ to highlight a context where $\A$ is intervened upon. In our applications to RL, we consider intervening only on a special variable $\A$ that represents the agent's actions. The truncation semantics~\Cref{eq:truncate} shows how an action changes the distribution over environment states and rewards. However, the intervention semantics is well-defined for manipulating {\em any} random variable in a causal model, not only a designated special action/decision variable. In terms of formal notation, the truncation semantics is well-defined for any intervention $\ado(\xv=\hat{\xvalue})$. In the example of the causal graphs of~\Cref{fig:cbns}, causal effects are defined for taking a shot ($\ado(SH=1)$), moving the player away from the goal ($\ado(CG=1)$), injuring the player ($\ado(PH=1)$), and ensuring that the attack ends in a goal ($\ado(SG=1)$).

The next lemma states that if the parent values of a manipulated variable are given, causal effects are equivalent to conditional probabilities.

\begin{lemma} \label{lemma:parent-condition}
    Let $\cn$ be a causal Bayesian network and let $\set{\yv},A,\set{\xv}$ be a disjoint set of random variables such that $\set{\xv} \supseteq \PA_{\A}$. Then $\mprob{\cn}(\set{\yv}|\set{\xv}=\set{\xvalue},\ado(\A = \hat{\a})) = \mprob{\cn}(\set{\yv}|\set{\xv}=\set{\xvalue},\A = \hat{\a})$. 
\end{lemma}

The lemma is easily derived from the do-calculus~\citep{Pearl2000};~\Cref{sec:observe-proof} gives a proof directly from the truncation semantics. The significance of the lemma is that if we can observe all the direct causes of an agent's actions, which we argue is the case in online RL, then the causal effect of an action can estimated by a conditional probability. 

Let $\set{\rv} = (\set{\ov} \cup \set{\lv})$ be a partition of the random variables into a nonempty set of observed variables $\set{\ov}$ and a set of latent variables $\set{\lv}$.  A set of observed variables $\set{\ov}$ is \defterm{causally sufficient} for variable $\xv$ in a causal graph $\G$ if all parents of $\xv$ are observed; that is, if $\pa_{\xv} \subseteq \set{\ov}$.\footnote{In~\Cref{sec:causal-suffice} below we discuss a weaker notion of causal sufficiency that is commonly used in causal discovery algorithms~\citep{Spirtes2000}.} A set of observable variables $\set{\ov}$ is \defterm{action sufficient} in a causal graph $\G$ if it is causally sufficient for the intervention variable $\A$. 
~\Cref{lemma:parent-condition} implies that if an observation observation signal $\set{\ov}$ is action sufficient in (the graph of) a CBN $\cn$, then conditional on the observations, causal effects of $\A$ are equivalent to conditional probabilities:

\begin{lemma} \label{lemma:observe-condition}
    Let $\set{\ov} \subseteq \set{\rv}$ be an action sufficient set of observable variables in a causal Bayesian network $\cn$. Then $\mprob{\cn}(\set{\yv}|\set{\ov}=\set{\ovalue},\ado(\A = \hat{\a})) = \mprob{\cn}(\set{\yv}|\set{\ov}=\set{\ovalue},\A = \hat{\a})$ for any list of target outcomes $\set{\yv}$. 
\end{lemma}

\subsection{Examples of Causal Bayesian Networks.} \label{sec:cbn-example}

We compute observational and interventional probabilities for taking a shot in the 
two causal graphs of~\Cref{fig:cbns}. All nodes are binary. We specify the following conditional probability parameters for the graph structures. 
For nodes without parents 
we assume a uniform prior:

\begin{equation*} \label{eq:prior}
    P(PH = 1) = P(CG = 1) = P(GH = 1) = 1/2
\end{equation*}

where variables are abbreviated with their initials (e.g. $\mathit{Player Health} = PH$).  

For illustration, we assume an unrealistically simple noise-free causal mechanism governing the player's behavior: The player shoots if and only if they are healthy and close to the goal. In symbols, we have

\begin{equation} \label{eq:behavioral}
    P(SH =1| PH, CG) =\begin{cases}
			1, & \text{if $PH = CG = 1$}\\
            0, & \text{otherwise}
		 \end{cases}
\end{equation}

A player scores if and only if they shoot, are healthy, close to the goal, and the goalie is not healthy. In symbols, we have

\begin{equation*}
    P(SC =1| SH, CG, GH) =\begin{cases}
			1, & \text{if $SH = CG = 1$ and $GH = 0$}\\
            0, & \text{otherwise}
		 \end{cases}
\end{equation*}

This parametrization implies the following {\em joint probability that the goalie is not healthy and all other variables are true}: 

\begin{equation} \label{eq:cbn-example-obs}
    \mprob{}(1=\mathit{PH} = \mathit{CG}  = \mathit{SH} = \mathit{SC}, \mathit{GH} = 0) = 1/2 \cdot 1/2 \cdot 1 \cdot 1 \cdot 1/2 = 1/8.
\end{equation}

We next compute the {\em causal effect of shooting on goal scoring, given that all other observable variables are true}. ~\Cref{tab:probs} gives the resulting scoring probabilities for seeing a player take a shot (observation) vs. intervening to take a shot (action).
~\Cref{fig:first-compute,fig:confound-compute}
show how these probabilities are computed.

\begin{table}[H]
\centering
\caption{Goal scoring probabilities derived from the models of~\Cref{fig:cbns}, given the evidence that all variables observable in the model are true. For observational probabilities, the player is observed to take a shot. For interventional probabilities, the shot is the result of an intervention. Note that observational and interventional probabilities differ only in the confounded offline model~\Cref{fig:confound}.}
\begin{tabular}{|l|c|c|l|}
\hline
                           & \begin{tabular}{c}
                            Observation  \\ $P(\mathit{SC}=1|\set{\ov=1},\mathit{SH}=1)$
                           \end{tabular} & \begin{tabular}{c}
                           Intervention  \\ 
            $P(\mathit{SC}=1|\set{\ov=1},\ado(\mathit{SH}=1))$
                           \end{tabular} \\ \hline
Online Model $\ref{fig:first}$    &    1/2   & 1/2         \\ \hline
Offline Model $\ref{fig:confound}$ & 1/2         & 1/4         \\ \hline
\end{tabular}
\label{tab:probs}
\end{table}

In the online model~$\ref{fig:first}$, 
Player Health (PH) is observed, so the shooting effect queries compared are $$P(\mathit{SC}=1|\mathit{CG}=1,\mathit{PH}=1,\mathit{SH}=1) = 1/2 \mbox{ and } P(\mathit{SC}=1|\mathit{CG}=1, \mathit{PH}=1,\ado(\mathit{SH}=1)) = 1/2.$$ 
Since both parents of the manipulated variable $\mathit{SH}$ are observed,~\Cref{lemma:parent-condition} implies that both the observational and interventional probabilities should be the same. The first row of~\Cref{tab:probs} confirms that both probabilities are 1/2; see~\Cref{fig:first-compute}.

In the offline model~$\ref{fig:confound}$, 
Player Health (PH) is {\em not} observed, so the shooting effect queries compared are $$P(\mathit{SC}=1|\mathit{CG}=1,\mathit{SH}=1) = 1/2 \mbox{ and } P(\mathit{SC}=1|\mathit{CG}=1, \ado(\mathit{SH}=1)) = 1/4, $$ 
which are different according to the second row of~\Cref{tab:probs}. As shown in~\Cref{fig:confound-compute}, the fundamental reason is that for observational probabilities, we can infer from the given observations that $\mathit{PH} = 1$; formally $P(\mathit{PH}=1|\mathit{CG}=1,\mathit{SH}=1))=1$. The truncation semantics removes the causal link $\mathit{PH} \rightarrow \mathit{SC}$, thereby {\em blocking the inference from effect to cause.}
%
%
 This example illustrate the general reason for why, given time stamps, confounders are the only case in which observational and interventional probabilities differ: If the common cause between action and reward is observed, it induces a non-causal correlation between action and reward, but conditioning on the common cause eliminates the non-causal correlation. The non-causal correlation can be eliminated only by observing the common cause; which is impossible if it is a latent confounder. 

\begin{figure}[bpth]

\begin{minipage}[b]{\textwidth} 
     \centering
     \begin{subfigure}[b]{0.45\textwidth}
         \centering
\includegraphics[width=\textwidth]{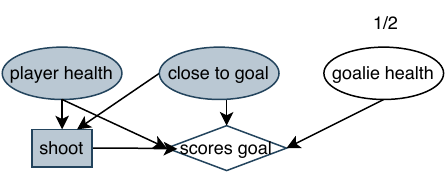}
         \caption{Observation}
         \label{fig:first-observe}
     \end{subfigure}
      \begin{subfigure}[b]{0.45\textwidth}
         \centering
         \includegraphics[width=\textwidth]{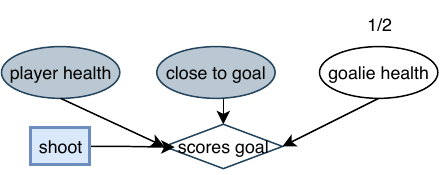}
         \caption{Intervention}
         \label{fig:first-intervene}
     \end{subfigure}
     \caption{Observational and intervention probabilities in the online model of~\Cref{fig:first}. Gray indicates observed variables whose values are specified in the query. Numbers indicate posteriors over latent variables, given the observations. Light blue indicates an intervention on a variable.~\Cref{fig:first-observe}: The observational scoring probability $P(\mathit{SC}=1|\mathit{CG}=1,\mathit{SH}=1,\mathit{PH}=1)$ is 1/2, the same as the probability that the goalie is healthy. ~\Cref{fig:first-intervene}: The query $P(\mathit{SC}=1|\mathit{CG}=1,\mathit{PH}=1,\ado(\mathit{SH}=1))$ is evaluated in the intervention model. The scoring probability remains 1/2, because both player health and shooting are observed, so breaking the causal link between them has no effect on the scoring probability. 
     \label{fig:first-compute}}
      \end{minipage}

\vspace{1cm} 

\begin{minipage}[bpth]{\textwidth} 
     \centering
     \begin{subfigure}[b]{0.45\textwidth}
         \centering
\includegraphics[width=\textwidth]{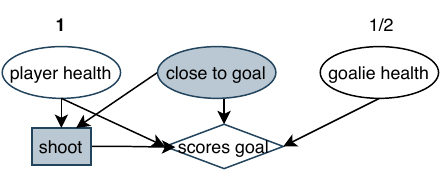}
         \caption{Observation}
         \label{fig:confound-observe}
     \end{subfigure}
      \begin{subfigure}[b]{0.45\textwidth}
         \centering
         \includegraphics[width=\textwidth]{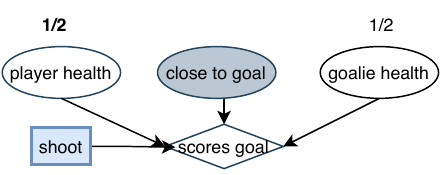}
         \caption{Intervention}
         \label{fig:confound-intervene}
     \end{subfigure}
     \caption{Observational and intervention probabilities in the confounded offline model of~\Cref{fig:confound}. The query $P(\mathit{SC}=1|\mathit{CG}=1,\mathit{SH}=1)$ is evaluated in the observation model~\Cref{fig:confound-observe}. {\em If we see a player shooting, we can infer that they are healthy.} Therefore the player scores if and only if the goalie is not healthy, so the scoring probability is 1/2. The query $P(\mathit{SC}=1|\mathit{CG}=1,\ado(\mathit{SH}=1))$ is evaluated in the intervention model~\Cref{fig:confound-intervene}. Without a link between player health and shooting, the probability of player and goalie health are both 1/2, which means that the scoring probability is 1/4.\label{fig:confound-compute}}
      \end{minipage}

\vspace{1cm} 

\end{figure}

\section{Background: Markov Decision Processes} \label{sec:pomdp}

As we explained in the introduction, for temporal data the difference between causation and correlation stems from the possible presence of confounders---unobserved common causes of the agent's actions and other environment variables. The appropriate setting for studying causality in RL is therefore a setting in which parts of the environment may be unobserved, which is known as a \emph{partially observable Markov decision process} (POMDP). In this section we review the basic theory of POMDPs.

\begin{table}[hbpt]
\centering
\caption{Correspondence between Causal and RL terminology.}
\begin{tabular}{cc}
\toprule
\begin{tabular}{l}
     \textbf{Reinforcement} \\
     \textbf{Learning}
\end{tabular} & \textbf{Causality} \\ 
\midrule
action                          & treatment          \\
reward                          & response           \\
observed state $O$ & observed co-variates variables $\set{\ov}$ \\
state        $\S$                   & co-variates $\set{\ov} \cup \set{\lv}$       \\
belief state $b(\S)$ & latent variable distribution $\prior(\set{\lv})$ \\
belief state update $b(S|O)$ & abduction $\prior(\set{\lv}|\set{\ov})$ \\
complete observability & causal sufficiency \\
\bottomrule
\end{tabular}
\label{table:rl-causality}

\end{table}


POMDP theory and causal concepts share a common formal structure, despite differences in terminology for describing interventions and their effects.~\Cref{table:rl-causality} shows translations between analogous concepts. Latent variables concepts are described in~\Cref{sec:counter}. Key differences include the following.

\begin{enumerate}
    \item In causal models, the response is a variable to be predicted, not a reward to be maximized.
    \item A reinforcement learning policy guides {\em sequential} decisions, not a single one-time treatment.
    \item RL concepts are usually defined in terms of a single state $s$; causal concepts are defined in terms of values for a list of variables. Using the terminology of~\cite[Ch.2.4.7]{Russell2010}, RL uses an {\em atomic} environment representation, whereas causal models use a {\em factored} representation.
\end{enumerate}

Bandit problems are the RL setting for one-time decisions, so causal models are closely related to contextual bandit problems~\citep{lattimore2016causal,lee2018structural}. Recent work has explored dynamic treatment regimes for applications in medicine, which brings causal modelling closer to the sequential setting of RL~\citep{zhang2020designing}. While RL theory and notation utilizes an atomic state representation $s$, factored representations are familiar in practice. For example in a grid world, a location is described as a coordinate pair $(x,y)$. Following previous causal models for RL~\citep{bib:barenboim-tutorial},\citep[Ch.17.4.3]{Russell2010}, we use a factored representation  
for the state space of a POMDP, as we describe next. We begin with the definition of a Markov Decision Process (MDP), then generalize it to POMDPs. 


\subsection{Factored Markov Decision Processes}  \label{sec:factored}


A \defterm{factored Markov Decision Process} (MDP) $P_{\E}$ is defined by the following components. 

\begin{itemize}
    \item {\em Variables:} A finite set of \defterm{state variables} $\set{\sv}$, an \defterm{action variable} $\A$ ranging over a finite set of actions available to the agent, and a real-valued \textbf{reward variable} $R$. An environment \defterm{state} is a assignment $\set{\svalue}$ to the state variables. 
    \item An {\em initial state distribution} $P_{\E}(\set{\svalue}_0)$
    \item A stationary {\em transition model} $P_{\E}(\set{\svalue}_{t+1}|\set{\svalue}_t,\a_t)$
    \item A stationary \emph{reward model} $P_{\E}(\r_{t}|\set{\svalue}_{t},a_{t})$
\end{itemize}

The environment is Markovian in the sense that the new state at time $t+1$ depends only on the current state and action, and is independent of previous states. It is well-known that 
in principle an  environment process can be converted to a Markov process by including the state-action history in the current state~\citep{Sutton1998}.

\paragraph{Example} In the example of~\Cref{fig:cbns}, the action and reward variables are $\A = \mathit{SH}$ and $\R = \mathit{SC}$. The state variables are $\set{\sv} = \{\mathit{CG}, \mathit{GH}, \mathit{PH} \}$. There are therefore $2^3$ environment states. For example, the assignment $\set{\sv} = \set{\svalue} \equiv \langle \mathit{CG} = 1, \mathit{GH} = 1, \mathit{PH} = 1\rangle$ represents the state where all variables are true. The {\em reward model} $P_{\E}((\mathit{SH}|\set{\sv},\A)) = \mprob{\ddn}(\mathit{SH}|\set{\sv},\A)$ was defined by the CBN of~\Cref{sec:cbn-example}. Informally, a player scores if and only if they shoot, are healthy, close to the goal, and the goalie is not healthy. Similarly, the {\em initial state distribution} is given by the CBN distribution over the $\set{\rv}$ variables. For instance, for the initial state $\set{\svalue}_0 = (\mathit{PH}=1,\mathit{CG}=1,\mathit{GH=1})$, we have $P_{\E}(\set{\svalue_0})= \mprob{\ddn}(\set{\svalue_0})= 1/2 \cdot 1/2 \cdot 1/2 = 1/8$.

\begin{figure}
    \centering
    \includegraphics{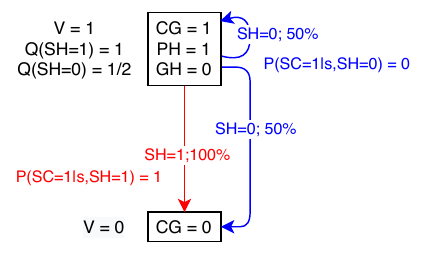}
    \caption{State diagram for the sports example illustrating on-policy policy evaluation in the online setting of ~\Cref{fig:first} with complete observability. We use state abstraction so that for example the abstract state labelled $\PH=0$ represents all states where the player is not healthy. For each state we specify its value V and Q action values, given the behavioral policy of shooting if and only if the player is healthy and close to the goal (\Cref{eq:behavioral}). Transitions are labelled with probabilities and annotated with expected rewards given a state-action pair. The discount factor $\gamma = 1$.}
    \label{fig:state-diagram}
\end{figure}

\Cref{fig:state-diagram} illustrates the {\em state transition model}. We treat the player and goalie health as persistent time-invariant features~\citep{lu2018deconfounding}:
$
 P_{\E}(\mathit{PH'}=\mathit{PH}) = 1 \mbox{ and }  P_{\E}(\mathit{GH'} = \mathit{GH}) = 1.$
 
For the Close-to-Goal state variable, we adopt the following transition model
\begin{equation*}
    P(\mathit{CG'} = 0| \mathit{SH}, \mathit{CG}) =
    \begin{cases}
			1, & \text{if $\mathit{CG}=0$ or $\mathit{SH} = 1$}\\
            1/2, & \text{if $\mathit{CG}=1$ and $\mathit{SH} = 0$}
		 \end{cases}
\end{equation*}

This model can be interpreted as follows. (1) 
We make the simplifying assumption that the attacking team stays close to the goal if and only if they maintain possession. Any state with $\mathit{CG}=0$ is therefore an absorbing state: $P_{\E}(\mathit{CG'}=0|\mathit{CG}=0)=1.$ (2) If the attacking team is close to the goal, our scenario works as follows.

\begin{itemize}
    \item If the player shoots, the team loses possession, either because they scored, or because the shot was blocked and the defending team took over (no rebounds):
    $P_{\E}(\mathit{CG'}=0|\mathit{SH}=1)=1.).$
    \item If the player does not shoot (e.g., they pass instead), there is a 50\% change that the attacking team retains possession.
\end{itemize}

A \defterm{policy} $\policy: \set{\sv} \rightarrow \Delta(\A)$ maps a state to a probability distribution over actions; we also write $\policy(\a|\set{\sv})$. The value function $V^{\policy}$ and action-value function $Q^{\policy}$ give the expected cumulative reward of a policy. They satisfy the \defterm{Bellman equation for policy evaluation}:

\begin{align} \label{eq:bellman-basic}
    \Q^{\policy}(\set{\svalue},\a) =  R(\set{\svalue},\a) + \gamma \sum_{\set{\svalue}'} P_{\E}(\set{\svalue}'|\set{\svalue},\a) V^{\policy}(\set{\svalue}')  \\
    \V^{\policy}(\set{\svalue}) = \sum_{a} \policy(a|\set{\svalue}) Q(\set{\svalue},a) \nonumber\\
    R(\set{\svalue},\a) = \sum_{\r} P_{\E}(\r|\set{\svalue},\a)   \nonumber \label{eq:xreward}
\end{align}

where $\gamma \in (0,1]$ is a discount factor.~\Cref{fig:state-diagram} illustrates the state and action value functions for our example, based on the policy defined by~\Cref{eq:behavioral}. {\em In all our sports examples, we assume that the policy of~\Cref{eq:behavioral} is the behavioral policy that generates the sports transition data.}

\begin{figure}[hbtp]
     \centering
     \begin{subfigure}[b]{0.30\textwidth}
         \centering
         \includegraphics[width=\textwidth]{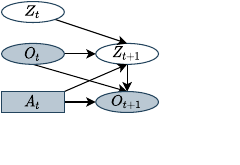}
         \caption{Transition model in factored POMDP.}
         \label{fig:pomdp-observe}
     \end{subfigure}
      \begin{subfigure}[b]{0.30\textwidth}
         \centering
         \includegraphics[width=\textwidth]{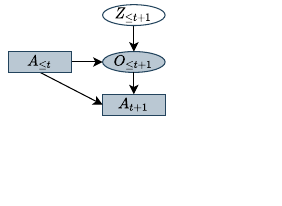}
         \caption{An executable policy}
         \label{fig:act-state}
     \end{subfigure}
      \begin{subfigure}[b]{0.30\textwidth}
         \centering
         \includegraphics[width=\textwidth]{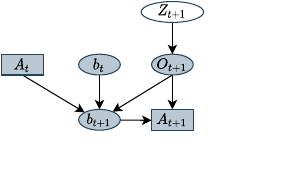}
         \caption{Transition model in belief MDP.}
         \label{fig:belief-mdp}
     \end{subfigure}
     \caption{Generic causal graphs for POMDPs.~\Cref{fig:pomdp-observe}: The environment transitions to a new latent state and generates an observation signal based on the current state and most recent action.
     ~\Cref{fig:act-state}: Given the action and observation history, the agent's current decisions are independent of the environment's latent state. The agent's policy can therefore be executed given their observation signal.
     ~\Cref{fig:belief-mdp}: The agent selects an action through their policy based on their current belief state $b$. Their updated belief state depends on the previous belief state, previous action and new observation.
     \label{fig:pomdp-graphs}}
\end{figure}

\subsection{Partially Observable Markov Decision Processes} \label{sec:pomdps}

A common causal model for MDPs are confounded MDPs~\citep{zhang2016markov}; see~\Cref{sec:ope} below. As noted by ~\citep{bruns2021model}, they can be seen as an instance of a POMDP. We use POMDP terminology as it is familiar to RL researchers. 

A \defterm{factored POMDP} is a factored MDP $P_{\E}$ together with a state variable partition

\[ \set{\sv} = \set{\ov} \cup \set{\lv}.\]

Here $\set{\ov}$ represents the set of observable state variables, also called \defterm{the observation signal}, and $\set{\lv}$ represents the latent state variables, which we sometimes refer to simply as the latent state. If $\set{\lv}$ is empty, the state is \defterm{completely observable}; otherwise it is \defterm{partially observable}. Since a pair $(\set{\lvalue},\set{\ovalue}) \equiv \set{\svalue}$ describes a state, we freely apply MDP notation to both $\set{\svalue}$ and pairs $(\set{\lvalue},\set{\ovalue})$. For example, POMDP components are described as follows.

\begin{itemize}
\item The initial state distribution $P_{\E}(\set{\svalue}_0)$ can be factored into a distribution over latent and observed variables:
\[
P_{\E}(\set{\svalue}_0) \equiv P_{\E}(\set{\lvalue}_0,\set{\ovalue}_{0}) = P_{\E}(\set{\lvalue}_0) \times P_{\E}(\set{\ovalue}_{0}|\set{\lvalue}_{0})
\]
where $P_{\E}(\set{\lvalue}_0)$ is an {\em initial latent distribution} and $P_{\E}(\set{\ovalue}_{0}|\set{\lvalue}_{0})$ is the {\em initial observation model.} 
 \item We assume that the  \emph{transition model} factors into a {\em latent update model} and a {\em dynamic observation model} (see~\Cref{fig:pomdp-observe}): 
 

\[
P_{\E}(\set{\svalue}_{t+1}|\set{\svalue}_t,\a_t) \equiv P_{\E}(\set{\lvalue}_{t+1},\set{\ovalue}_{t+1}|\set{\svalue}_t,\a_t) = P_{\E}(\set{\lvalue}_{t+1}|\set{\svalue}_{t},\a_t) \times P_{\E}(\set{\ovalue}_{t+1}|\set{\lvalue}_{t+1},\set{\ovalue}_{t},\a_{t}).
\]
\end{itemize}


In the standard atomic POMDP formulation, the observation model $P(o|s,a)$ depends on the entire state~\cite[Ch.17.4]{Russell2010}. However in factored representation, the entire state $\set{\svalue}$ includes the current observations $\set{\ovalue}$. Instead we assume that the current observations are Markovian in that they depend only on the current latent state, and the most recent observation and action:
$$P_{\E}(\set{\ovalue}_{t+1}|\set{\lvalue}_{t+1},\set{\ovalue}_{t},\a_{t})= P_{\E}(\set{\ovalue}_{t+1}|\set{\lvalue}_{t+1},\set{\ovalue}_{\leq t},\a_{\leq t})$$
%
%

\paragraph{Example.} In the online model of ~\Cref{fig:first}, the observable state variables are $\set{\ov} = \{\mathit{CG},\mathit{PH}\}$. Therefore the assignment $(\set{\ov} = \set{\ovalue}) = \langle CG = 1, PH = 1\rangle$ is the observation signal $\set{\ovalue}$ received by the learning agent in the state $\set{\svalue}$ where all variables are true. In the offline model of~\Cref{fig:confound}, the observation signal does not include the player health. Thus $\set{\ov} = \{\mathit{CG}\}$ and $\set{\ovalue} = \langle CG = 1\rangle$ is the observation signal received by the learning agent in the state $\set{\svalue}$ where all variables are true. 

The difference illustrates the fundamental fact, highlighted by~\cite{Zhang2020,zhang2016markov,bruns2021model}, that {\em different agents can receive different observation signals in the same environment state}. The reason is that the observations depend on not only the state of the environment, but also on the perceptual capabilities of the agents. In particular the agent learning from an offline dataset may not have access to the same observations as the behavioral agent whose behavior generated the offline dataset (see also~\Cref{fig:partial-offline}). Because the agent-relative distinction between observable and latent state variables is key for causal modelling in RL, we highlight it with another example adapted from~\citep{Zhang2020}. They consider a self-driving car scenario similar to that illustrated in~\Cref{fig:driving}. In the online setting, the agent learns by driving themselves; their observation signal includes the tail light of the car in front of them. In the offline setting, the agent learns from a dataset collected by drone surveillance, which does not include the front tail lights. For the task of learning an optimal policy that maximizes reward, {\em a restriction on the agent's perceptual abilities translates into a restriction on the space of policies that are feasible for them.} In the driving example, the expert's policy may include a rule such as ``brake if I see a tail light in front of me''. If the learning agent does not observe the tail lights, this policy is not available to them.

\begin{figure*}
     \centering
     \begin{subfigure}[b]{0.40\textwidth}
         \centering
\includegraphics[width=\textwidth]{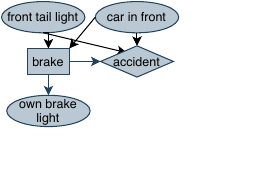}
         \caption{Online view}
         \label{fig:car-online}
     \end{subfigure}
      \begin{subfigure}[b]{0.40\textwidth}
         \centering
         \includegraphics[width=\textwidth]{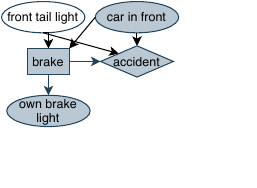}
         \caption{Offline view}
         \label{fig:car-offline}
     \end{subfigure}
     \caption{An illustration of different observation signals in the same environment, based on an example by~\cite{Zhang2020}. 
     ~\Cref{fig:car-online}: In the online setting, the driving agent learns about traffic by driving themselves. Their braking decisions are influenced by observing the taillight of the car in front of them.
     ~\Cref{fig:car-offline}: In the offline setting, the driving agent learns from a driving demonstration dataset. The dataset includes drone data specifying the location of other cards, but not the status of their tail lights. 
     \label{fig:driving}}
\end{figure*}

To formalize the notion of a  policy feasible for an observation signal $\set{\ov}$, we note that the interaction between an environment and an agent generates \defterm{transition data} of the form $\tau = \set{\lvalue}_0,\set{\ovalue}_0,\a_0,\r_0,\set{\lvalue}_1,\set{\ovalue}_1,\a_1,r_1,\ldots,\set{\lvalue}_t, \a_t,r_{t},\set{\lvalue}_{t+1},\set{\ovalue}_{t+1},\a_{t+1}$~\citep{Sutton1998}: 
Environment state $\set{\svalue}_t \equiv \set{\lvalue}_t,\set{\ovalue}_t$ occurs, an action $\a_t$ is chosen by the agent's policy, resulting in a reward $\r_{t}$ and next state $\set{\lvalue}_{t+1}$ according to the environment dynamics. At the next time step, another action $\a_{t+1}$ is chosen, and so on. 
A policy $\policy$ is \defterm{executable with observation signal $\set{\ov}$} if for all times $t$, the agent's actions are independent of the latent state (see~\Cref{fig:act-state}):

\[\policy(\a_{t}|\set{\lvalue}_{t},\set{\ovalue}_{\leq t}, \set{\a}_{<t}) = \policy(\a_{t}|\set{\ovalue}_{\leq t}, \set{\a}_{<t}) \]

Our notion of an executable policy is essentially equivalent to Russell and Norvig's concept of an {\em agent function}, which maps a history of the agent's percepts and actions to a distribution over current actions~\cite[Ch.2.1]{Russell2010}. In our sports example, the policy of shooting if and only if the player is close to the goal and healthy is executable in the online model of~\Cref{fig:first} and not executable in the offline model of~\Cref{fig:confound}, where player health is not observable. 

To find an optimal executable policy $\policy$ under partial observability, the most common framework is to transform the POMDP into a belief MDP. The next section describes this transformation for a factored POMDP.

\subsection{The Belief MDP} \label{sec:belief-mdp}

Even if the environment dynamics is Markovian in the state space, it may not be Markovian in observation space,
because past observations can and typically do carry information about the current latent state. In order to apply MDP techniques to a POMDP, the most common approach is to transform the POMD into an equivalent MDP whose states represent the agent's current beliefs. A \defterm{belief state} is a distribution $b(\set{\lv})$ over the latent environment state. The basic idea is that a POMDP can be transformed into an MDP by replacing latent states with the agent's {\em beliefs} about latent states. As 
\cite[Ch.17.4.1.]{Russell2010} write in their standard textbook: ``The fundamental insight...is this: {\em the optimal action depends only on the agent's current belief state}'' (emphasis \citeauthor{Russell2010}). Note that while an agent's belief state $b_t$ is {\em about} the latent state $\set{\lvalue}_t$, the latent state does not cause their beliefs, so a policy based on beliefs is executable (see~\Cref{fig:belief-mdp}). 

In a factored POMDP, observations are separate from the latent space, and the agent's decisions depend not only on their current belief state, but also on their current observation. Accordingly, an \defterm{epistemic state} $\langle\set{\ovalue},b \rangle$ comprises a current observation $\set{\ovalue}$ and a current belief state. We use the tuple notation $\langle \rangle$ to make longer formulas easier to parse, and to emphasize the analogy between epistemic states and MDP states in a traditional atomic representation. An {\em executable policy} maps an epistemic state to a distribution over actions:

\begin{equation*} \label{eq:policy}
\policy: \set{\ov}^{\policy} \times B^{\policy} \rightarrow \Delta(\A)
\end{equation*}
where $\set{\ov}^{\policy}$ represents the \defterm{observation space} of the agent executing policy $\policy$ and $B^{\policy}$ the space of belief states, i.e. distributions $\Delta(\set{\lv}^{\policy})$ over the unobserved variables. We also use the conditional probability notation  $\policy(\a|\langle \set{\ovalue},b \rangle) $. 
We can view belief states as summarizing the observation-action history, much as hidden states summarize past sequences in a recurrent neural network. 

\paragraph{Belief Dynamics} Given a current action $\a_t$ and observation ${\set{\ovalue}}_{t+1}$, the agent's beliefs move from current beliefs $b_{t}$ to updated beliefs $b_{t+1}$ through posterior updates; see~\Cref{fig:belief-mdp}. In a belief MDP model, the interaction between an agent's policy and the environment dynamics gives rise to a sequence $\tau = \set{\lvalue}_0,\set{\ovalue}_0,b_0,\a_0,\r_1,\set{\lvalue}_1,\set{\ovalue}_1,b_1,\a_1,r_2,\ldots,\set{\lvalue}_t,\set{\ovalue}_t,b_t,\a_t,r_{t+1},\set{\lvalue}_{t+1},\set{\ovalue}_{t+1},b_{t+1},\a_{t+1},\ldots$.

The observation signal provides an agent with information about the latent environment state through the posterior distribution $P(\set{\lvalue}_{t+1}|\set{\ovalue}_{\leq t+1},\a_{\leq t})$. As we explain in~\Cref{sec:counter}, posterior updates are also a key operation in computing counterfactual causal effects. We next derive the well-known POMDP formula for recursively updating the latent state posterior~\cite[Ch.17.4.2]{Russell2010} for a factored POMDP and an executable policy. 


\begin{observation} \label{obs:posterior} If the transition data are generated by a POMDP $P_{\E}$ and an executable policy, the latent posterior update is given by
\begin{align} \label{eq:update-b}
    P(\set{\lvalue}_{t+1}|\set{\ovalue}_{\leq t+1},\a_{\leq t}) = \alpha P_{\E}(\set{\ovalue}_{t+1}|\set{\lvalue}_{t+1},\set{\ovalue}_{t},\a_{t}) \times \sum_{\set{\lvalue}_{t}} P_{\E}(\set{\lvalue}_{t+1}|\set{\lvalue}_{t},\set{\ovalue}_{t},\a_t) P(\set{\lvalue}_{t}|\set{\ovalue}_{\leq t},\a_{<t}) 
\end{align}
where $\alpha$ is a normalization constant.
\end{observation}

\begin{proof}
\begin{align*}
P(\set{\lvalue}_{t+1}|\set{\ovalue}_{\leq t+1},\a_{\leq t}) \propto P_{\E}(\set{\ovalue}_{t+1}|\set{\lvalue}_{t+1},\set{\ovalue}_{\leq t},\a_{\leq t}) \sum_{\set{\lvalue}_{t}} P_{\E}(\set{\lvalue}_{t+1}|\set{\lvalue}_{t},\set{\ovalue}_{\leq t},\a_{\leq t}) P(\set{\lvalue}_{t}|\set{\ovalue}_{\leq t},\a_{\leq t}) \mbox{ by Bayes' theorem} \\
= P_{\E}(\set{\ovalue}_{t+1}|\set{\lvalue}_{t+1},\set{\ovalue}_{t},\a_{t}) \sum_{\set{\lvalue}_{t}} P_{\E}(\set{\lvalue}_{t+1}|\set{\lvalue}_{t},\set{\ovalue}_{t},\a_{t}) P(\set{\lvalue}_{t}|\set{\ovalue}_{\leq t},\a_{\leq t})  \mbox{ applying the Markov property} \\
= P_{\E}(\set{\ovalue}_{t+1}|\set{\lvalue}_{t+1},\set{\ovalue}_{t},\a_{t}) \sum_{\set{\lvalue}_{t}} P_{\E}(\set{\lvalue}_{t+1}|\set{\lvalue}_{t},\set{\ovalue}_{t},\a_{t}) P(\set{\lvalue}_{t}|\set{\ovalue}_{\leq t},\a_{<t})  \mbox{ because $\policy$ is executable}
\end{align*}
\end{proof}

The standard notation for an agent's posterior \defterm{belief state} at time $t$ is $b_t(\set{\lvalue}_{t+1}) \equiv P(\set{\lvalue}_{t+1}|\set{\ovalue}_{\leq t+1},\a_{\leq t})$.~\Cref{eq:update-b} shows how the new belief state $b_{t+1}$ can be computed from the previous belief state $b_t$.
We adopt the standard POMDP notation for an agent's current belief state $b$ and $b'$ for a successor belief state. Similarly, we write $\set{\ovalue,\ovalue'}$ for an observation signal and its successor, and $\set{\lvalue,\lvalue'}$ for a latent state component and its successor. With these conventions, 
the belief state update~\Cref{eq:update-b} becomes

\begin{align} \label{eq:update-b'}
    b'(\set{\lvalue'}) = \alpha P_{\E}(\set{\ovalue}'|\set{\lvalue}',\set{\ovalue},\a) \times E_{\set{\lvalue}\sim b(\set{\lvalue})} [P_{\E}(\set{\lvalue'}|\set{\lvalue},\set{\ovalue},\a)]
\end{align}

~\Cref{eq:update-b'} is analogous to the standard POMDP for atomic POMDPs, with the latent state variables $\set{\lvalue}$ replacing the latent state $\svalue$. 


\paragraph{Policy Evaluation for Belief MDPs}

The \defterm{policy evaluation task} is to compute the value function for a \defterm{policy} $\policy$ in a given environment. 
The Bellman equation to evaluate a policy based on epistemic states is as follows. 

\begin{align}
    \Q^{\policy}(\langle \set{\ovalue},b \rangle,\a) =  R(\langle \set{\ovalue},b \rangle,\a) + \gamma \sum_{\set{\ovalue}'} P(\set{\ovalue}'|\langle \set{\ovalue},b \rangle,\a) V^{\policy}(\set{\ovalue}',b') \label{eq:observe-bellman} \\
    \V^{\policy}(\langle \set{\ovalue},b \rangle) = \sum_{a} \policy(a|\langle \set{\ovalue},b \rangle) Q(\langle \set{\ovalue},b \rangle,a) \nonumber\\
    R(\langle \set{\ovalue},b \rangle,\a) = E_{\set{\lvalue}\sim b(\set{\lvalue})} 
    R(\set{\lvalue},\set{\ovalue},\a)  \label{eq:xreward}\\
    P(\set{\ovalue}'|\langle \set{\ovalue},b \rangle,\a) = \sum_{\set{\lvalue'}}P_{\E}(\set{\ovalue}'|\set{\lvalue'},\set{\ovalue},\a) E_{\set{\lvalue}\sim b(\set{\lvalue})} 
    [P_{\E}(\set{\lvalue}'|\set{\lvalue},\set{\ovalue},\a)]  \label{eq:new-observe}
\end{align}


where $\gamma \in (0,1]$ is a discount factor.




According to the recurrent equation~\Cref{eq:observe-bellman}, given a new observation $\set{\ovalue'}$ and the current action $\a$, the expected policy value can be computed in two steps:

\begin{description}
    \item[Posterior Update] Compute the new belief state $b'$ by conditioning the current beliefs $b$ on observations $\set{\ovalue},\set{\ovalue}'$ and action $\a$  following~\Cref{eq:update-b'}.
    \item[Prediction] Estimate the expected return $V^{\policy}(\set{\ovalue}',b')$ given the new observation $\set{\ovalue}'$ and new belief state $b'$. 
\end{description}

Updating a posterior to predict the outcome of an action $\a$ is a key part of the formal semantics of counterfactuals that we present in~\Cref{sec:counter}.~\Cref{fig:state-diagram-pomdp} illustrates how the Bellman equation can be used to evaluate our behavioral policy in the online POMDP setting where Goalie Health is not observable. In the next section we illustrate policy evaluation in the offline model of~\Cref{fig:confound}, where the evaluation is based on a causal model. 

\begin{figure}
    \centering
    \includegraphics{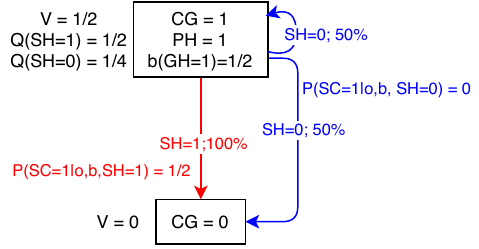}
    \caption{On-policy policy evaluation in the online setting of ~\Cref{fig:first} with partial observability. The policy evaluated is the standard behavioral policy. An epistemic state comprises values for the observable variables $\CG$ and $\PH$ and a belief over the values of the latent variable $\GH$. We use state abstraction so that for example the abstract state labelled $\CG=0$ represents all states where the attacking team is not close to the goal. The policy evaluated is the standard behavioral policy of~\Cref{eq:behavioral}, which chooses to shoot if and only if the player is  healthy and close to the goal. The diagram shows the V value and Q action values for the epistemic state where the agent is observed to be healthy and close to the goal, and their belief is uniform over the latent variable $\GH$. The evaluation uses reward and transition probabilities derived from the dynamic model of~\Cref{sec:factored}.  Transitions are labelled with probabilities. State-action pairs are annotated with expected rewards. The discount factor $\gamma = 1$.}
    \label{fig:state-diagram-pomdp}
\end{figure}

 In a typical RL setting, we evaluate a policy $\policy$ learned from data generated by a \defterm{behavioral} policy $\policy_{\beta}$; see~\Cref{fig:offline-picture}.\footnote{The learned policy is called the estimation policy in~\cite[Ch.5.6]{Sutton1998}, the target policy by~\cite{wan2021learning}, and the evaluation policy by~\cite{bruns2021model}.} In the \defterm{on-policy} setting, the policy to be evaluated is the same as the policy generating the data, so $\policy = \policy_{\beta}$. In the next section we discuss how a learned policy can be evaluated using a dynamic causal model.

\section{Dynamic Decision Networks for POMDPS} \label{sec:dim}

In this section we describe a causal Bayesian network for POMDPS. Following~\cite[Ch.17]{Russell2010} and ~\cite{bib:cooper-yoo,Boutilier1999}, we utilize a {\em dynamic causal Bayesian network}. Dynamic Bayesian networks extend BNs to temporal data. The basic idea is to make a copy $\set{\rv}'$ for the random variables in the BN, to represent successor variables. The dynamic BN is then a BN over the current and successor variables (i.e., over $\set{\rv} \cup \set{\rv'}$), such that there are no edges from the successor variables to the current variables. A dynamic BN satisfies the Markov condition in that the successor variables depend only on variables at a previous time, Adding action and reward variables to a dynamic CBN defines a {\em dynamic influence diagram}, a widely adopted graphical formalism for Markov decision processes~\citep{polich2007interactive}. We follow~\cite{Russell2010} and use the term \defterm{dynamic decision network} (DDN) instead of the term dynamic influence diagram. Their terminology emphasizes a DDN is special kind of CBN, so the concepts and results of~\Cref{sec:bns} apply.


\begin{definition} \label{def:ddn} A \defterm{dynamic decision network} (DDN) $\ddn$ for state variables $\set{\sv}$ comprises the following random variables.
\begin{enumerate}
    \item Current time slice: $\set{\V} = \set{\sv} \cup \{\A\} \cup \{R\} \cup \{B\}$
    \item Next time slice: $\set{\V'} = \set{\sv'} \cup \{\A'\} \cup \{R'\} \cup \{B'\}$.
\end{enumerate}
A DDN $\ddn$ satisfies the following causal assumptions. 
    \begin{enumerate}
    \item There are no edges from $\set{\rv}'$ to $\set{\rv}$.
    \item There are no edges from $\R$ to nodes in $\set{\rv}$, and no edges from $\R'$ to nodes in $\set{\rv}'$.
    \item There are no edges from $\A$ to nodes in $\set{\rv} - \{\R\}$, and no edges from $\A'$ to nodes in $\set{\rv}' - \{\R'\}$.
    \end{enumerate}
\end{definition}

These assumptions state that (1) causal relationships respect the temporal ordering, (2) rewards may causally depend on the current state and action, but not vice versa, (3) actions may causally depend on the current state, but not vice versa, and not on the current reward.~\Cref{fig:pomdp-graphs} illustrates these assumptions using generic causal graphs. The driving example~\Cref{fig:driving} does not satisfy Assumption 3, because the state variable 
``own brake light'' causally depends on the action of braking.~\cite{de2019causal} argue that this is possible when the temporal resolution of events is low enough that ``braking'' and ``brake light'' are assigned the same discrete time index. In that case conditioning on the state signal entails conditioning on an effect of the action, leading to ``causal confusion''~\citep{de2019causal}. We discuss this scenario further under related work. Our main conclusions do not depend on Assumption 3, but we use it to simplify formal arguments. 
We next give an extended example of a DDN for our sports scenario. 

\subsection{Dynamic Decision Network Example} \label{sec:ddn-example}

\begin{figure}
    \centering
    \includegraphics{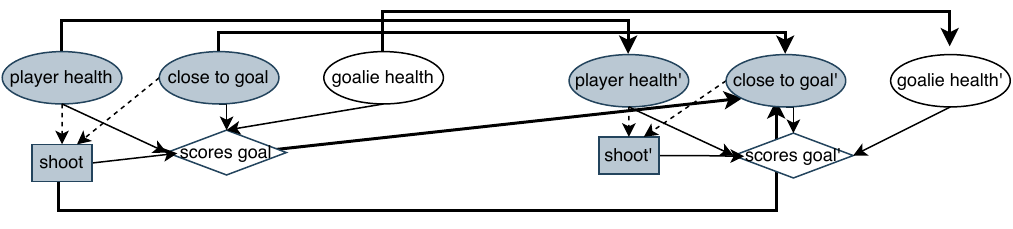}
    \caption{A dynamic decision network graph for our sports scenario. The reward model is indicated by thin lines, and the state transition model by thick lines. Dashed lines represent the agent's policy. (The agent's belief state is not shown, see text).}
    \label{fig:dnn}
\end{figure}

A fully specified influence diagram, or DDN, defines a joint distribution $\mprob{\ddn}(\set{\rv},\set{\rv'})$ over both time slices. Therefore {\em a DDN specifies both an environment MDP and a behavioral policy}. We illustrate how in the example DNN structure of \Cref{fig:dnn}. As with any CBN, the parameters of a DDN are conditional probabilities of the form $P(\rvalue_i|\pa_i)$. It is straightforward to parametrize the DDN graph to match the POMDP model of~\Cref{sec:pomdp}. For example, the {\em reward model} $\mprob{\ddn}(\mathit{SH}|\set{\sv},\A)$ was defined in~\Cref{sec:cbn-example}. 


For the {\em state transition model}, we have
\begin{align*}
 \mprob{\ddn}(\mathit{PH'}=\mathit{PH}) = 1 \mbox{ and }  \mprob{\ddn}(\mathit{GH'} = \mathit{GH}) = 1\\
  \mprob{\ddn}(\mathit{CG'} = 0| \mathit{SH}, \mathit{CG}) =
    \begin{cases}
			1, & \text{if $\mathit{CG}=0$ or $\mathit{SH} = 1$}\\
            1/2, & \text{if $\mathit{CG}=1$ and $\mathit{SH} = 0$}
		 \end{cases}
\end{align*}

The \defterm{network  policy} $\policy^{\ddn}$  is defined by the conditional distributions $\policy^{\ddn}(\A|\PA_{\A})$ and $\policy^{\ddn}(\A'|\PA_{\A'})$. In our example, we have 

\begin{align*}
  \mprob{\ddn}(\mathit{SH} = 1| \CG,\PH) =
    \begin{cases}
			1, & \text{if $\mathit{CG}=1$ and $\PH = 1$}\\
            0, & \text{otherwise}
		 \end{cases}
\end{align*}

Throughout the examples in this paper, but not in our general theory, we assume that the network policy is Markovian in that it depends only on the current state variables, but not the agent's current beliefs.
Our theorems and analysis address the general case of non-Markovian policies that can depend on the agent's entire observation history. We adopt the policy Markov assumption in our examples for simplicity: A belief state is a continuous object, so writing out an explicit interpretable mapping from a belief state to a decision requires too much detail for illustrative purposes.~\cite[Ch.17.4]{Russell2010} provide a worked out example. Models of posterior beliefs for decision-making have been developed in the literature~\citep{Hausknecht2015,Liu2020PlayerEmbedding}, but they are usually based on deep learning and not straightforward to interpret. 


\subsection{Action Sufficiency and Policy Executability}

A DDN represents the agent's current model of the transition data, which may or may not be accurate. We say that a DDN $\ddn$ \defterm{matches an environment model} $P_{\E}$ if the network reward/state transition/initial model agree with $P_{\E}$. For example, matching requires that $\mprob{\ddn}(\R|\set{\sv}_0))=P_{\E}(\R|\set{\sv}_0)$ for every reward value $\R$ and initial state value $\set{\sv}_0$. 

Given a set of observable DDN variables $\set{\ov}$ and $\set{\ov'}$, 
a DDN policy $\policy^{\ddn}$ \defterm{matches a behavioral policy} $\policy_{\beta}: \set{\sv} \rightarrow \Delta(\A)$ if $\policy(\A|\set{\sv}) = \mprob{\ddn}(\A|\set{\sv}\cap \PA_{\A})$ and $\policy(\A'|\set{\sv'}) = \mprob{\ddn}(\A'|\set{\sv'}\cap \PA_{\A'})$. Thus matching requires that, given the parents of the action variable, other state variables are independent of the action.


Intuitively, a network policy is executable if all causes of the actions are observable, which means that the network is action sufficient. Extending the definition of action sufficiency from ~\Cref{lemma:parent-condition} to DDNs, we say that a set $\set{\ov}$ of observable state variables is \defterm{action sufficient} in a DDN $\ddn$ if it includes all causes of decisions other than the agent's beliefs; that is, $\pa_{\A} \subseteq \set{\ov}^{\ddn} \cup \{B\}$ and  $\pa_{\A'} \subseteq \set{\ov'}^{\ddn} \cup \{B'\}$. 
The next observation states that a DDN satisfies action sufficiency if and only if actions are independent of latent variables. We also require a minor technical definition that rules out redundant parents: We say that a CBN is {\em action-minimal} if for every parent $\xv$ of action variable $\A$ and every set $\set{U}$ of variables disjoint from $\A$ and $\xv$, we have $P(\A|\xv,\set{U})\neq P(\A|\set{U})$. That is, there is no variable set $\set{U}$ such that conditioning on $\set{U}$ makes $\xv$ independent of its child $\A$. Local minimality is entailed by the well-known stability/faithfulness conditions~\cite[Ch.2.4]{Pearl2000}. 

\begin{observation} Let $\ddn$ be a locally minimal dynamic decision network. Then a set $\set{\ov}$ of observable variables is action sufficient in $\ddn$ if and only if the network policy $\policy^{\ddn}$ is executable given $\set{\ov}$. 
\end{observation}
\begin{proof}
     $(\Rightarrow)$: Suppose that every parent of $\A$ is observed (i.e., $\PA_{\A} \subseteq \set{\ov})$. By~\Cref{def:ddn}(3), the only potential descendant of $\A$, except for successor variables, is the reward variable $\R$. Thus the set of observed variables $\set{\ov}$ contains no descendant of $\A$. 
    By the Markov condition, $\A$ is independent of all non-descendants given the parents of $\A$. So $\A$ is independent of all contemporaneous latent environment variables given $\set{\ov}$, which is the definition of an executable policy in~\Cref{sec:pomdps}. The same argument applies to the successor action variable $\A'$. 
    
    $(\Leftarrow)$: Suppose that $\A$ is independent of the latent variables $\set{\lv}$ given the observed variables $\set{\ov}$. Then action-minimality requires that no latent variable is a parent of $\A$, which is the definition of action-sufficiency.
\end{proof}

The upshot is that executable policies can be represented by an action sufficient DDN, where all parents of the action variable are observable.

If we allow that an action can causally affect state variables (i.e., do not assume~\Cref{def:ddn}(3)), we still have that executability implies action-sufficiency, which suffices for our main argument that in online RL, conditional and causal probabilities coincide. The next section develops this argument formally.

\subsection{Policy Evaluation with a Dynamic Decision Network}

Since a DDN defines an environment process, we can use it to perform {\em model-based evaluation of a policy}, by computing the required reward and transition probabilities from the DDN. The resulting \defterm{DDN Bellman equation} is derived from~\Cref{eq:observe-bellman} as follows. 

\begin{align}
    \Q^{\policy,\ddn}(\langle \set{\ovalue},b \rangle,\a) =  R^{\ddn}(\langle \set{\ovalue},b \rangle,\a) + \gamma \sum_{\set{\ovalue}'} \mprob{\ddn}(\set{\ovalue}'|\langle \set{\ovalue},b \rangle,\a) V^{\policy,\ddn}(\set{\ovalue}',b') \label{eq:observe-ddn} \\
    \V^{\policy,\ddn}(\langle \set{\ovalue},b \rangle) = \sum_{a} \policy(a|\langle \set{\ovalue},b \rangle) \Q^{\policy,\ddn}(\langle \set{\ovalue},b \rangle,a) \nonumber\\
    R^{\ddn}(\langle \set{\ovalue},b \rangle,\a) = \sum_{r} \r \cdot \mprob{\ddn}(\R = \r|\set{\ovalue},b,\a)  \nonumber
\end{align}

Since a DDN defines a joint distribution over both current and successor states, the required conditional probabilities are also specified by the model.~\Cref{fig:state-diagram-pomdp} shows the evaluation of the sports DDN for our standard behavioral policy (\Cref{eq:behavioral}). To illustrate the DDN Bellman equation in the offline model of~\Cref{fig:confound}, note that our standard behavioral policy is not executable in this model. Instead we use~\Cref{eq:observe-ddn} to evaluate the \defterm{marginal} policy $\marginal$ derived from the behavioral policy averaging over latent states. The marginal policy is an important concept in offline policy evaluation~\citep{kausik2024offline,bruns2021model}. The marginal policy can be viewed as a naive form of behavioral cloning where we estimate the agent's action probability from frequencies based on the observation signals. In our example, if the player is close to the goal and we do not know their health status, the probability that they take a shot is 1/2. So we have the following standard marginal policy for our sports example:

\begin{equation} \label{eq:marginal}
    P(\SH =1|\CG) =\begin{cases}
			1/2, & \text{if $\CG = 1$}\\
            0, & \text{otherwise}
		 \end{cases}
\end{equation}

~\Cref{fig:state-diagram-offline} shows the Q and V values for the epistemic state where $\CG=1$. Given the reward and state-transition probabilities derived from the DDN model of~\Cref{sec:ddn-example}, the epistemic state value $V$ satisfies the equation $V = 1/4 + 1/4 V$, so $V = 1/3$. 
This computation is an example of {\em model-based off-policy evaluation}: Our DDN is a joint model of the behavioral agent and their environment that potentially can be constructed from observed transition data. We then use the DDN model of the behavioral policy to evaluate another executable policy, in this case the marginal behavioral policy of~\Cref{eq:marginal}. 

\begin{figure}
    \centering
    \includegraphics{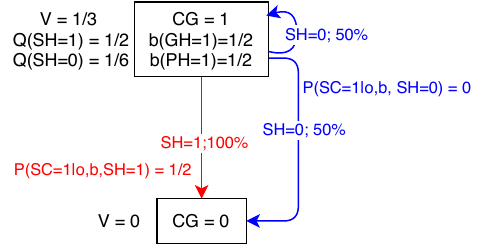}
    \caption{Off-policy policy evaluation based on conditional probabilities in the offline setting of~\Cref{fig:confound}. An epistemic state comprises values for the observable variable $\CG$ and a belief over the values of the latent variables $\GH$ and $\PH$. The policy evaluated is the marginal behavioral policy of~\Cref{eq:marginal}, which chooses to shoot with probability 1/2 if the player is close to the goal. The evaluation uses {\em conditional} reward and transition probabilities derived from the DDN of~\Cref{sec:ddn-example}. The diagram shows the V value and Q action values for the epistemic state where the agent is observed to be close to the goal, and their belief is uniform over the latent variables. Transitions are labelled with probabilities. State-action pairs are annotated with expected rewards.}
    \label{fig:state-diagram-offline}
\end{figure}

\section{Causal Effects and Online Policy Evaluation} \label{sec:intervene}

{\em Interventional policy evaluation} views as a policy as selecting an intervention $\ado(\hat{a})$, not only selecting an action $\a$~\citep{wang2021provably,zhang2020designing}). The fundamental question of interventional policy evaluation in RL is when interventional reward-transition probabilities and Q-values can be inferred from conditional probabilities. In this section we prove that {\em they are equivalent for any action sufficient dynamic influence model.} In online RL, the agent is aware of the causes of their actions, which means that the data generating process follows a model that is causally sufficient for their actions. We argue this point in detail in the following~\Cref{sec:obs-equiv}. Our overall conclusion is that in online RL, observational and interventional probabilities are equivalent. Our first proposition considers the reward and transition models.

\begin{proposition} \label{prop:online-reward} Suppose that an observation signal $\set{\ov}$ is action sufficient in a dynamic decision network $\ddn$. Then $\mprob{\ddn}(\R|\set{\ov},\hat{\A}) = \mprob{\ddn}(\R|\set{\ov},\ado(\hat{\A}))$ and $\mprob{\ddn}(\set{\sv}'|\set{\ov},\hat{\A}) = \mprob{\ddn}(\set{\sv}'|\set{\ov},\ado(\hat{\A}))$. \end{proposition}
\begin{proof}
   Follows immediately from~\Cref{lemma:observe-condition} (with $\set{\ov} \cup \{B\}$ in place of $\set{\ov}$). 
\end{proof}

\Cref{tab:probs} illustrates how conditional and interventional reward probabilities are the same for the action sufficient model of~\Cref{fig:first} and different from the confounded model of~\Cref{fig:confound}. 

The next proposition asserts that for action sufficient models, both conditional and interventional value and action value functions are the same. The \defterm{Bellman equation for the interventional Q-function} is obtained by replacing in the observational Bellman equation conditioning on an action $\a$ by conditioning on the intervention $\ado(\hat{a})$ (cf. \cite{wang2021provably,zhang2020designing}). For instance, the interventional version of the DDN~\Cref{eq:observe-ddn} is as follows:

\begin{align}
    \Q^{\policy,\ddn}(\langle \set{\ovalue},b \rangle,\ado(\hat{a})) =  R^{\ddn}(\langle \set{\ovalue},b \rangle,\ado(\hat{a})) + \gamma \sum_{\set{\ovalue}'} \mprob{\ddn}(\set{\ovalue}'|\langle \set{\ovalue},b \rangle,\ado(\hat{a})) V^{\policy,\ddn}(\set{\ovalue}',b') \label{eq:intervene-ddn} \\
    \V^{\policy,\ddn}(\langle \set{\ovalue},b \rangle) = \sum_{a} \policy(a|\langle \set{\ovalue},b \rangle) \Q^{\policy,\ddn}(\langle \set{\ovalue},b \rangle,\ado(\hat{a})) \nonumber\\
     R^{\ddn}(\langle \set{\ovalue},b \rangle,\ado(\hat{a})) = E_{\set{\lvalue}\sim b(\set{\lvalue}|\ado(\hat{a}))} \sum_{r} \r \cdot \mprob{\ddn}(\R = \r|\set{\lvalue},\set{\ovalue},\ado(\hat{a})) \nonumber
\end{align}



\begin{proposition} \label{prop:online-value} Suppose that set of variables $\set{\ov}$ is action sufficient in a dynamic decision network $\ddn$. Then for any policy $\policy$ with observation signal $\set{\ov}$, the observational $Q$-value equals the interventional $Q$-value for every epistemic state and action:
\begin{equation*}
 \Q^{\policy,\ddn}(\langle \set{\ovalue},b \rangle,\hat{a}) =  \Q^{\policy,\ddn}(\langle \set{\ovalue},b \rangle,\ado(\hat{a})). 
\end{equation*}
\end{proposition}

\begin{proof}
The basic insight is that conditional and interventional probabilities agree on rewards and transitions by~\Cref{prop:online-reward}.  The resulting value functions then agree as well because they are defined recursively by reward and transition probabilities. Formally, consider $$E_{\set{\lvalue}\sim b(\set{\lvalue})} R(\set{\lvalue},\set{\ovalue},\ado(\hat{a})) = \sum_{\r} \mprob{\ddn}(\r|\set{\ovalue},\ado(\hat{a})) = \sum_{\r} \mprob{\ddn}(\r|\set{\ovalue},\hat{a}) = E_{\set{\lvalue}\sim b(\set{\lvalue})} R(\set{\lvalue},\set{\ovalue},\hat{a})$$ 

where the penultimate equality follows from~\Cref{prop:online-reward}. The same proposition implies that the state transition probabilities are the same for conditional and international probabilities. Since reward and transition probabilities are the same, so are the value functions. \end{proof}

Although~\Cref{prop:online-value} is stated in terms of a causal model, it does not assume that the learning agent is given a true causal model of the behavioral policy and the environmental process. Rather, the import is that as long as the learning agent's observation signal includes the causes of the behavioral agent's actions, conditional probabilities inferred from the action data are equivalent to interventional probabilities, no matter what the true dynamic causal model is. 

\paragraph{Examples.} To illustrate the proposition for the standard behavioral policy $\policy$, we have from ~\Cref{fig:state-diagram-pomdp} that $\Q^{\policy,\ddn}(\CG=1,\PH=1,b, \SH=1) = 1/2$; we leave it to the reader to verify that also $\Q^{\policy,\ddn}(\CG=1,\PH=1,b, \ado(\SH=1))=1/2$. 

For the marginal behavioral policy $\marginal$ in the offline model,~\Cref{fig:intervene-diagram-offline} shows the Q and V values for the epistemic state where $\CG=1$. Given the reward and state-transition probabilities derived from the DDN model of~\Cref{sec:ddn-example}, the epistemic state value $V$ satisfies the equation $V = 1/2 \cdot 1/4 + 1/4 V$, so $V = 1/6$. For the interventional action values, we have

$$\Q^{\ddn}(\CG=1,b,\ado(\SH=1)) = 1/4.$$
In contrast, from ~\Cref{fig:state-diagram-offline} we have the conditional action value $$\Q^{\ddn}(\CG=1,b,\SH=1) = 1/2.$$

Intuitively, the reasoning for the difference is as follows. If we {\em observe} that a player is close to the goal and takes a shot, the DDN model entails that they are healthy. 
The value of a state where the player is healthy, close to the goal, and shoots is 1/2, as shown in~\Cref{fig:state-diagram}. However, if we {\em intervene} to make the player take a shot, the association between player health and shooting is broken, and the updated belief that the player is healthy remains 1/2 in the DDN model. Since the player has a chance of scoring only if they are healthy, their scoring chance is $1/2 \cdot 1/2 = 1/4.$

\begin{figure}
    \centering
    \includegraphics{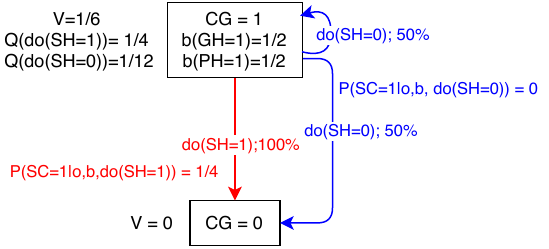}
    \caption{Off-policy policy evaluation based on interventional probabilities in the offline setting of~\Cref{fig:confound}. The policy evaluated is the marginal behavioral policy of~\Cref{eq:marginal}, which chooses to shoot with probability 1/2 if the player is close to the goal. The evaluation uses {\em interventional} reward and transition probabilities derived from the offline DDN. The diagram shows the V value and Q action values for the epistemic state where the agent is observed to be close to the goal, and their belief is uniform over the latent variables.}
    \label{fig:intervene-diagram-offline}
\end{figure}

\section{Online Reinforcement Learning, Observation-equivalence, and Action Sufficiency} \label{sec:obs-equiv}

We complete our analysis by examining which RL learning settings can be expected to satisfy action sufficiency. The main condition we consider is online learning. Our chain of argument can be summarized as follows; the remainder of this section defines the relevant concepts and unpacks the steps. 

\begin{quote}
    Online Learning $\Rightarrow$ Observation-equivalence  $\Rightarrow$ Action Sufficiency  \\ $\Rightarrow$ Interventional Probabilities = Conditional Probabilities
\end{quote}




\paragraph{Online Learning $\Rightarrow$ Observation-equivalence.} 

The {\em behavioral policy}  $\policy_{\beta}$ interacts with the environment to generate transition data, whereas the {\em learned policy} $\policy$ is constructed from the data to maximize the agent's return. In a belief or epistemic MDP,  we can represent each as a function mapping an observation signal and belief over latent states to a distribution over actions:

\begin{align*} \label{eq:policy}
\policy: \set{\ov}^{\policy} \times B^{\policy} \rightarrow \Delta(\A) \\
\policy_{\beta}: \set{\ov}^{\beta} \times B^{\beta} \rightarrow \Delta(\A)
\end{align*} 


Alternatively, we can view a policy as mapping a sequence of observations and actions to a distribution over the current action, as in Russell and Norvig's {\em agent function} ~\cite[Ch.2.1]{Russell2010}. In either case, the {\em space of policies that are executable for an agent depends on their observation signal.} In causal terms, the agent's observation signal causes their decisions. An important insight for causal modelling is therefore that the observation signals for the behavior and the learned policy can be different~\citep{Zhang2020}. We say that the behavioral and learned policy are \textbf{observation-equivalent} if they share the same observation space, that is, $\set{\ov}^{\policy} = \set{\ov}^{\beta}$. 





In the sports example of~\Cref{fig:confounded}, the behavior policy is executed by the athlete, who is influenced by their health. Their observation signal comprises the set $\set{\ov}^{\beta} = \{\mathit{Player Health, Close to Goal}\}$. The learned policy of the coach is based on the publicly available match data, which includes player locations but not player health. Their observation signal comprises the set $\set{\ov}^{\policy} = \{\mathit{Close to Goal}\}$. 
~\cite{Zhang2020} illustrate the concept of different observation spaces in a self-driving car scenario similar to that illustrated in~\Cref{fig:driving}. In the online scenario, the agent learns by driving themselves; their observation signal includes the tail light of the car in front of them. In the offline scenario, the agent learns from a dataset collected by drone surveillance, which does not include the front tail lights.


In the {\em online setting}, the learning agent executes a behavior policy to learn about the environment {\em from their own experience}. The observation signal that drives the behavior policy is therefore the same as the observation signal, so we have observational equivalence and $\set{\ov}^{\policy} = \set{\ov}^{\beta}$. 
The causes of the behavior policy actions are included in the $\set{\ov}^{\policy}$, and are therefore accessible to the learning agent in online learning. Note that this conclusion applies to both on-policy and off-policy learning (cf.~\Cref{fig:online-picture}). For a simple example, suppose that the agent uses an $\epsilon$-greedy policy for exploration, where the behavioral policy $\policy_{\beta}$ chooses a random action with probability $\epsilon$ and follows the current policy $\policy$ with probability $1-\epsilon$. While the $\epsilon$-greedy policy $\policy_{\beta}$ is not the same as the current policy $\policy$, both are based on the same observations.  

\paragraph{Observation-equivalence $\Rightarrow$ Action Sufficiency}

Observation-equivalence implies that every cause of the observed actions is observed by the learning agent. Therefore every common cause of the actions and another variable, such as reward or the next state, is observed by the learning agent. Actions and rewards are therefore not confounded by a latent variable, and interventional probabilities can be inferred from causal probabilities, as shown  in~\Cref{sec:intervene}. 

The online setting is sufficient but not necessary for observation-equivalence. Other sufficient conditions are on-policy learning, where $\policy_{\beta} = \policy$, and complete observability, as discussed in~\Cref{sec:intro}.
In sum, the sufficient conditions for observation equivalence we have discussed are as follows:

\begin{quote}
Online Learning/Complete Observability/On-policy Learning $\Rightarrow$ Observation-equivalence \\
\end{quote}


\section{Structural Causal Models and Counterfactuals} \label{sec:counter}


\begin{table}[htb]
\centering
\caption{Two counterfactual probabilities relevant to RL. The random variable $R'$ denotes the actual reward received, whereas $\R'_{\hat{\A}}$ denotes the potential reward following the decision to perform action $\hat{\A}$. \label{table:reward}}
\begin{tabular}{|c|p{5cm}|l|}
\hline
Notation                              & Reading                             & Query Type \\ \hline
\begin{tabular}{l}
$P(\R'_{\hat{\A}}|\set{\S},\A)$ \mbox{or}  \\
$P(\R'|\set{\S},\A,\ado(\hat{\A})$ 
\end{tabular}   & What is the reward if we were to choose $\hat{A}$ instead of $\A$?  & what-if    \\ \hline
$P(\R'_{\hat{\A}}|\R',\set{S},\A))$ & If we receive reward $R'$ after choosing action $\A$, what is the reward after choosing $\hat{A}$ instead? & hindsight  \\ \hline
\end{tabular}

\end{table}

We return to counterfactuals, the most advanced kind of causal reasoning according to Pearl.~\Cref{table:reward} gives two generic examples of counterfactual reward queries relevant to RL. In a {\em hindsight query}, we observe an outcome, and ask what the {\em potential outcome} is from an alternative course of action. In a {\em what-if query}, we do not observe the actual outcome, and ask how an alternative course of action might change what outcome is likely. What-if queries are closely related to policy optimization: Suppose that an agent's current policy recommends action $a$ in state $s$ (i.e., $\policy(s)=a$). To ascertain whether an alternative action $a'$ would improve the agent's current policy, we can ask what the reward distribution would be if the agent chose $a'$ instead in state $s$. We show that under action sufficiency, what-if counterfactuals are equivalent to conditional probabilities. Thus in the online setting, what-if queries can be answered using conditional probabilities, as is done in traditional RL.

While causal theory has focused on hindsight counterfactuals for evaluating potential outcomes as the most complex type of counterfactual query~\citep{Pearl2000}, hindsight counterfactuals have not played a major role in reinforcement learning. One reason for this is that traditional reinforcement learning is based on conditional probabilities, and hindsight counterfactuals cannot be reduced to conditional probabilities, even in online learning, as we show in~\Cref{sec:hindsight} below.  
However, the seminal work on hindsight experience replay~\citep{bib:hindsight} demonstrated that information about potential outcomes can be leveraged to improve RL. Recent work has proposed augmenting transition data with virtual experiences derived from hindsight counterfactuals~\citep{bib:hindsight-augment}. We discuss use cases for hindsight counterfactuals further in our related work~\Cref{sec:related}. 

In this section we introduce the formal semantics of counterfactual reasoning. We show how 
a {\em structural causal model} (SCM) can be used to evaluate observational, interventional, and counterfactual probabilities from a single model. 
SCMs combine a causal graph with local functions that map parent values to child values. They achieve greater expressive power than causal Bayesian networks by including latent 
variables as parents. 
The posterior over latent variables carries over information from observed outcomes to infer potential outcomes. 

\subsection{Structural Causal Models}

Structural causal models (SCMs) combine a causal graph with latent variables and deterministic local functions that map parent values to child values~\citep{Pearl2000}. While the use of latent variables typically makes them less interpretable than a Bayesian network with observable variables only, SCMs have two important advantages as causal models. First, the latent variables support a formal semantics for counterfactuals.
Second, SCMs are compatible with deep learning, in that 
the local functions are essentially decoders in the sense of deep generative models. Deep generative models with latent variables are therefore a powerful architecture for implementing and learning with counterfactual reasoning~\citep{geffner2022deep}. For our purpose of relating causal models and reinforcement learning, SCMs contain two important elements that match standard POMDP theory: unobserved components and a distribution over latent variables, which corresponds to an agent's beliefs in POMDP theory. 
 Our presentation of SCMs follows that of \cite{Pearl2000} and of \cite{scholkopf2021towards}, and highlights connections with other generative models in machine learning. 

 A  \defterm{structural causal model} (SCM) is a pair $\scm = \langle \G,\set{F} \rangle$ meeting the following conditions. 

 \begin{itemize}
    \item $\G$ is a DAG over random variables $\set{\rv}$. Let $\set{\uv}$ be the set of \defterm{source nodes} with indegree 0 in $\G$. We write $\set{\ns{\rv}} = \set{\rv} - \set{\uv}$ for the set of non-source variables. 
    \item $\set{F} = \{f_1,\ldots,f_n\}$ is a set of local functions such that each $f_i$ deterministically maps the parents of non-source variable $\rv_i \in \set{\ns{\rv}}$ to a value of $\rv_{i}$. 
    The local functions are often written in the form of a \defterm{structural equation}:
    \begin{equation} \label{eq:local}
         \rv_{i} = f_i(\pa_i).
    \end{equation}
\end{itemize}

The source variables are also called {\em background variables} or {\em exogeneous variables}. They can be, and often are, latent variables. For example, in a linear structural equation 
\[
 f(\xv,\varepsilon) \equiv \yv = a \xv + b + \varepsilon 
\]
the parent $\xv$ of $\yv$ represents an observed cause and the noise term $\varepsilon$ can be modelled as a latent parent summarizing unobserved causal influences. 
The set of functions $\set{F}$ takes the place of the conditional probability parameters in a causal Bayesian network. 
Given an assignment $\set{\uv} = \set{\uvalue}$ of values to the source variables, we can  compute values for the non-source variables $\set{\ns{\rv}}$ by starting with the source nodes, then assigning values to the children of the source nodes, etc. This recursive evaluation procedure defines a \defterm{solution function} 

\begin{equation}
    \label{eq:evaluate}
    F^{\scm}(\set{\uvalue}) = \set{\ns{\rvalue}}
\end{equation}

where $\set{\ns{\rvalue}}$ is an assignment of values to the non-source nodes $\set{\ns{\rv}}$. A \defterm{probabilistic SCM} is a pair $\cm = (\scm,\prior)$ where $\prior$ is a joint \defterm{prior distribution} over the  source variables $\set{\uv}$~\citep[Eq.7.2]{Pearl2000}. The source variables $\set{\uv}$ typically are latent, so we use the POMDP notation $\prior$ rather than the more usual $p(\set{\uv})$ to emphasize the similarity between a POMDP belief and a distribution over source variables. 
Like a causal Bayesian network, a probabilistic SCM $\cm$ defines a \textbf{joint distribution} over the variables $\set{\rv} = \set{\uv} \cup \set{\ns{\rv}}$: 

\begin{align} 
    \mprob{\cm}(\set{\ns{\rv}} = \set{\ns{\rvalue}}, \set{\uv} = \set{\uvalue} ) =  \mprob{\cm}(\set{\ns{\rv}} = \set{\ns{\rvalue}}|\set{\uv} = \set{\uvalue}) \times \prior(\set{\uv} = \set{\uvalue}) \label{eq:cm-observe} \\
    \prior(\set{\uv} = \set{\uvalue}) = \prod_{\uv \in \set{\uv}} \prior(\uv = \uvalue) \label{eq:factor-u} \\
     \mprob{\cm}(\set{\ns{\rv}} = \set{\ns{\rvalue}}|\set{\uv} = \set{\uvalue}) = 
    \begin{cases}
        1, & \text{if $F^{\scm}(\set{\uvalue}) = \set{\ns{\rvalue}}$}\\
            0, & \text{otherwise}
    \end{cases} \label{eq:cm-decode}
\end{align}

~\Cref{eq:factor-u} says that the unconditional prior distribution over source variables factors into individual priors over each source variable, which is true in any Bayesian network.~\Cref{eq:cm-decode} applies the deterministic solution function as a deterministic decoder that maps the source variables to a unique assignment for non-source variables.

\paragraph{Relationship to Causal Bayesian Networks} A well-known result in graphical model theory states that every Bayesian network $\cn$ over
$n$ variables $\set{\rv}$ can be represented by an equivalent SCM over variables $\set{\rv} \cup \{\uv_{i}, i=1,\ldots,n\}$ such that  $P(\rvalue_i|\pa_i) = \sum_{\uvalue_i:f_i(\pa_i,\uvalue_i) = \rvalue_i} \prior(\uvalue_i)$ where $f_i$ is the local SCM function, and $\prior(\uvalue_i)$ is the prior probability assigned to the $\uvalue_i$ value~\citep{druzdzel1993causality}. This parametrization of Bayesian networks associates a scalar latent variable $\uv_i$ with each observed variable $\rv_i$. The corresponding structural equation is $\rv_i = f_i(\pa_i,\uvalue_i)$. The latent scalar $\uv_i$ is often called the {\em error} or {\em noise} term for variable $\rv_i$, with common notations being $\lv_i$ or $\epsilon_i$.  The resulting SCM is equivalent to the Bayesian network $\cn$ in the sense that the joint SCM distribution defined by~\Cref{eq:cm-observe} is equivalent to the joint Bayesian network distribution defined by the product formula~\Cref{eq:bn-observe}. 
SCMs and causal Bayesian networks therefore have equivalent expressive power in terms of the joint distributions that they can represent. As ~\Cref{fig:scms} illustrates, the CBNs of~\Cref{fig:cbns} are also SCMs because their conditional probabilities are deterministic.

\begin{figure}
     \centering
     \begin{subfigure}[b]{0.32\textwidth}
         \centering
         \includegraphics[width=\textwidth]{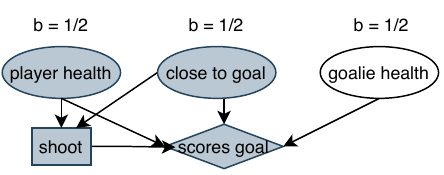}
         \caption{Online View.}
         \label{fig:first-scm}
     \end{subfigure}
      \begin{subfigure}[b]{0.32\textwidth}
         \centering
         \includegraphics[width=\textwidth]{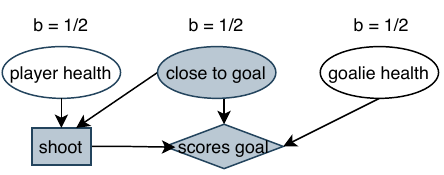}
         \caption{Offline View with confounder.\label{fig:confound-scm}}
     \end{subfigure}
     \begin{subfigure}[b]{0.32\textwidth}
         \centering
         \includegraphics[width=\textwidth]{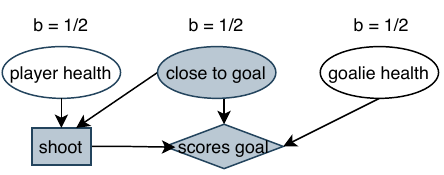}
         \caption{Offline view without confounder.}
         \label{fig:no-confounder}
     \end{subfigure}
     \caption{Structural Causal Models to illustrate what-if counterfactuals in our sports scenario. Number labels specify the prior distribution over source variables. 
     ~\Cref{tab:equations} specifies the deterministic local functions.
     ~\Cref{fig:first-scm}: The SCM version of the online model~\Cref{fig:first}. 
     ~\Cref{fig:confound-scm}: The SCM version of the offline model~\Cref{fig:confound}.~\Cref{fig:no-confounder}: An offline view where Player Health is not observed and does not affect goal scoring chances.
     \label{fig:scms}}
\end{figure}

\paragraph{Relationship to Encoder-Decoder Generative Models} If we restrict latent variables to comprise all and only source variables, then $\set{\ns{\rv}} = \set{\ov}$ is the set of observed variables, and~\Cref{eq:cm-decode} is the decoder model $P(\set{\ov} = \set{\ovalue}|\set{\lv} = \set{\lvalue})$ for generating observations from latent factors $\set{\uv} = \set{\lv}$. The factoring condition~(\ref{eq:factor-u}) then becomes an {\em independent component analysis} condition where observations are generated from independent sources. Similar independence conditions are used in deep generative models, such as the Variational Auto-encoder~\citep{khemakhem2020variational}. While ICA models are excellent density estimators for a distribution over observed variables, they are restrictive for causal modelling. For example, they assume that latent variables can only be causes, not effects, of observed variables. This is not true in many domains; for example, high blood pressure causes heart damage, even if the person with high pressure has no diagnostic tools for observing the heart damage. Fundamentally, the issue is that which factors are observable depend on the perceptual capabilities of an agent, whereas causal relationships among variables pertain to the environment and hold regardless of what an agent can observe. 
Another implication of the $\set{\uv} = \set{\lv}$ restriction is that latent variables are not causes of each other.  Indeed recent causal models include dependencies among latent variables. For example, the CausalVAE approach uses a linear structure equation $f_i$ with a Gaussian noise model for each non-source latent variable $Z_i$~\citep{yang2021causalvae}.~\cite{bib:brehmer} model dependencies among latent variables with neural networks. 



\paragraph{Structural Causal Model Examples} 

~\Cref{tab:equations} shows structural equations for the graphs of~\Cref{fig:first-scm,fig:confound-scm}, that represent the same causal mechanisms with deterministic functions that were described in~\Cref{sec:bns} with conditional probabilities. 
For the SCM of~\Cref{fig:no-confounder}, we use the same equation for $\mathit{SH}$ but replace the scoring equation by
\[ \mathit{SC} = \mathit{CG}  \cdot \mathit{SH} \cdot (1-\mathit{GH})       \]
since in this model Player Health does not affect scoring chances. 
The sharing of structural equations across scenarios illustrates how they represent local independent mechanism that we can expect to be stable across different contexts and domains. This is an important advantage of causal modelling for machine learning~\citep{scholkopf2021towards}.

Together~\Cref{fig:first-scm} and~\Cref{tab:equations} define a probabilistic structural causal model $\cm$. Adding a prior distribution  over the  binary source variables $\mathit{GH}, \mathit{PH}, \mathit{GH}$ specifies a probabilistic SCM. With a uniform distribution over the source variables, the observational joint distribution assigns

\begin{equation} \label{eq:scm-example-obs}
    \mprob{\cm}(1=\mathit{PH} = \mathit{CG}  = \mathit{SH} = \mathit{SC}, \mathit{GH} = 0) = 1/2 \cdot 1/2 \cdot 1 \cdot 1 \cdot 1/2 = 1/8.
\end{equation}

which agrees with the corresponding result for the causal Bayesian network from~\Cref{eq:cbn-example-obs}.




\begin{table}[htb]
\centering
\caption{A set of structural equations for the causal graphs of ~\Cref{fig:first-scm,fig:confound-scm}.}
\label{tab:equations}
\begin{tabular}{|l|l|}
\hline
Variable                      & Equation                                                                            
\\ \hline
$\mathit{Shoot (SH)}$         & $\mathit{SH} = \mathit{PH} \cdot \mathit{CG}$                                        \\ \hline
$\mathit{Scores (SC)}$        & $\mathit{SC} = \mathit{PH} \cdot \mathit{CG} \cdot \mathit{SH} \cdot (1-\mathit{GH})$ \\ \hline
\end{tabular}
\end{table}

\paragraph{Interventional Distributions}


Similar to causal Bayesian networks, for a structural causal model $\cm = (\scm,\prior)$ the distribution that results from intervening on a variable $\A$  is computed by removing the parents of $\A$ from the graph and replacing the local function $f_{\A}$ by a constant function. 
%
%
The \defterm{submodel} $\scm_{\hat{a}} = \langle \G_{\A},\set{F}_{\hat{a}} \rangle$ is the causal model where $\G_{\A}$ contains all edges in $\G$ except those pointing into variable $\A$, and $\set{F}_{x} = \{f_{i}: X_{i} \neq \X\} \cup \{X=\hat{a}\}$. Here $\{f_{i}: X_{i} \neq \A\}$ is the set of all local functions for unmanipulated variables, and $\A=\hat{a}$ is  the constant function that assigns variable $\A$ its manipulated value.
Similarly let $\prior_{\hat{a}}$ be the prior distribution over source node variables that assigns probability 1 to $\hat{a}$ and agrees with $\prior$ on all other variables. Formally, $\prior_{\hat{a}}(\A=\hat{a}) = 1$, and $\prior_{\hat{a}}(\uv=\uvalue) = \prior(\uv=\uvalue)$ for $\uv \neq \A$. We compute the intervention distribution as the joint probability in the truncated submodel:

\begin{equation} \label{eq:intervention}
 \mprob{(\scm,\prior)}_{\ado(\A = \hat{a})}(\set{\rv} = \set{\rvalue}) = \mprob{(\scm_{\hat{a}},\prior_{\hat{a}})}_{}(\set{\rv} = \set{\rvalue})
\end{equation}


~\Cref{fig:compute-confound} below illustrates the truncation semantics. We next show how the interventional distribution can be used to define a formal semantics for counterfactuals. 

\subsection{Causal Effects and Counterfactual Probabilities} 



A \defterm{counterfactual probability} 
$P(\set{\yv}_{\hat{a}}=\set{\yvalue'}|\set{\xv} = \set{\xvalue}, \A=\a,\set{\yv}=\set{\yvalue})$ can be read as follows: ``Given that we observed action $\A=\a$, and state variables $\set{\xv} = \set{\xvalue}$, followed by outcome $\set{\yv}=\set{\yvalue}$, what is as the probability of obtaining an alternative outcome $\set{\yvalue'}$, if we were to instead select the action $\hat{a}$ as an intervention?'' Here $\set{\yv}_{\hat{a}}$ is a list of \defterm{potential outcome} random variables, distinct from the actual outcomes $\set{\yv}$, and $\set{\xv},\A, \set{\yv}$ are disjoint. 
%
For a given probabilistic SCM $\cm = (\scm,\prior)$, we can compute the counterfactual probability 
as follows~\cite[Th.7.1.7]{Pearl2000}.

\begin{description}
    \item[Abduction/Posterior Update] Condition on the observations $\set{\xv} = \set{\xvalue},\A=\a,\set{\yv}=\set{\yvalue}$ to compute a source variable posterior $$\prior' \equiv \prior(\set{\uv}|\set{\xv}=\set{\xvalue}, \A=\a,\set{\yv}=\set{\yvalue}).$$ 
    \item[Intervention] Apply the intervention $\ado(\A=\hat{\a})$ to compute the submodel $\scm_{\hat{a}}$ and the SCM $\cm' = (\scm_{\hat{a}},\prior'_{\hat{a}})$. 
    \item[Prediction] 
    Return the conditional probability $P(\set{\yv}=\set{\yvalue'}|\set{\xv}=\set{\xvalue})$ computed in the updated SCM: 
\begin{align} \label{eq:scm-counterfactual}
    \mprob{\cm}(\set{\yv}_{\hat{a}}=\set{\yvalue'}|\set{\xv} = \set{\xvalue}, \A=\a,\set{\yv}=\set{\yvalue}) = \mprob{\cm'}(\set{\yv}=\set{\yvalue}'|\set{\xv} = \set{\xvalue})
\end{align}
\end{description}

A posterior update is often called abduction in causal modelling, and in the field of knowledge representation in general~\citep{poole1993probabilistic}. 
Through their posterior, the source variables carry information from the observed configuration $\set{\xv}=\set{\xvalue}, \A=\a,\set{\yv}=\y$ to the counterfactual configuration where $\set{\xv}=\set{\xvalue}, \A=\hat{\a},\set{\yv}=\set{\yvalue'}$. Counterfactual probabilities form a natural hierarchy that generalizes interventional probabilities as follows.

\begin{itemize}
    \item We refer to the most general counterfactual of the form $P(\set{\yv}_{\hat{a}}|\set{\xv}, \A,\set{\yv})$ as a {\bf hindsight} counterfactual query because it specifies the actual outcomes $\set{\yv}$.
    \item  If the actual outcomes $\set{\yv}$ are {\em not} included in the evidence, we have a \defterm{what-if} counterfactual query $P(\set{\yv}_{\hat{a}}|\set{\xv}, \A)$  that asks what the likely outcome is after deviating from the actual choice $\A$. For what-if queries, we use the causal effect notation $P(\set{\yv}|\set{\xv}, \A,\ado(\A=\hat{\a})) \equiv  P(\set{\yv}_{\hat{a}}|\set{\xv}, \A).$ 
    \item If neither an observed outcome nor an observed action are specified, a counterfactual probability reduces to an interventional probability $P(\set{\yv}|\set{\xv}, \ado(\A=\hat{\a}))$.
    \item If neither an observed outcome nor an intervention are specified, a counterfactual probability reduces to a conditional probability $P(\set{\yv}|\set{\xv},\A)$. 
\end{itemize}

We refer to probabilities of the first three types that involve interventions as \defterm{causal probabilities}. The update-intervention-prediction procedure can be used for computing any causal probability. 

\subsection{Examples for What-if Counterfactual Probabilities.}

We illustrate the computation of causal probabilities for the causal models of~\Cref{fig:scms}. In the online model~\Cref{fig:first-scm}, the scoring probability due to an intervention not to shoot is unaffected by observing the 
agent shoot, and therefore equal to just the interventional scoring probability, which is 0 (see~\Cref{sec:cbn-example}).  Formally, we have

\[
P(\SC=1|\CG=1,\PH=1,SH=1,\ado(\SH=0)) = P(\SC=1|\CG=1,\PH=1,\ado(\SH=0)) = 0.
\]

To verify this equality, note that in the online model, goalie health is the only latent variable, and it is independent of shooting, so the updated posterior is just the prior. 

For the offline models in~\Cref{fig:scms}, 
 our example what-if counterfactual asks for the {\em probability of scoring a goal if the player were to shoot, given that they actually did not shoot.} ~\Cref{tab:scm-queries} compares the corresponding reward probabilities given state and action information for observations, interventions, and counterfactuals. 
~\Cref{fig:compute-confound} 
illustrates the observation-intervention-prediction steps for the offline models. 
%
For the confounded model of~\Cref{fig:confound-scm}, the observational and interventional probabilities agree with our computations for the corresponding causal Bayesian network (see~\Cref{fig:confound-compute}).  Because the player did not shoot, we can infer that they are not healthy (abduction). Since they are not healthy, the what-if probability of scoring is 0 given the scoring model.

In the non-confounded model~\Cref{fig:no-confounder},  
observational, interventional, and counterfactual probabilities are the same because the posterior over the player health variable does not affect the goal scoring probability.~\Cref{fig:compute-non-confound} shows the computation of the what-if counterfactual. We leave the computation of the conditional and interventional probabilities as an exercise. In the next subsection we show that the equivalences of ~\Cref{fig:first-scm,fig:no-confounder} are an instance of a general pattern: {\em under action sufficiency, what-if counterfactuals are equivalent to conditional probabilities.} 

\begin{table}[H]
\centering
\caption{Scoring probabilities for the reward model in the soccer examples of~\Cref{fig:confound-scm,fig:no-confounder}.\label{tab:scm-queries}}
\resizebox{\textwidth}{!}{%
\begin{tabular}{|l|c|c|c|}
\hline
                            & Conditional                                     & Intervention                                         & What-If Counterfactual                                                    \\ \hline Query
                           & $P(\mathit{SC}=1|\mathit{CG}=1,\mathit{SH}=1)$ & $P(\mathit{SC}=1|\mathit{CG}=1,\ado(\mathit{SH})=1))$ & $P(\mathit{SC}=1|\mathit{CG}=1,\mathit{SH}=0,\ado(\mathit{SH})=1))$ \\ 
\hline Model $\cm_{\ref{fig:confound-scm}}$        & 1/2                                             & 1/4                                                  & 0                                                                  \\ \hline
Model $\cm_{\ref{fig:no-confounder}}$        & 1/2                                             & 1/2                                                 & 1/2                                                                  \\ \hline
\end{tabular}
}

\end{table}

\begin{figure}[H]
\begin{minipage}[b]{\textwidth} 
     \centering
     \begin{subfigure}[b]{0.32\textwidth}
         \centering
\includegraphics[width=\textwidth]{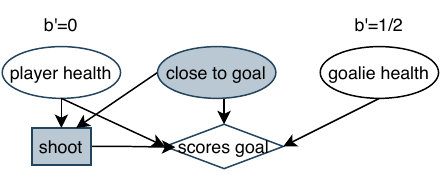}
         \caption{Abduction}
         \label{fig:observe-confound}
     \end{subfigure}
      \begin{subfigure}[b]{0.32\textwidth}
         \centering
         \includegraphics[width=\textwidth]{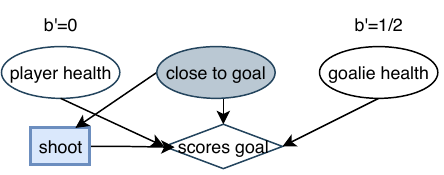}
         \caption{Intervention}
         \label{fig:intervene-confound}
     \end{subfigure}
     \begin{subfigure}[b]{0.32\textwidth}
         \centering
\includegraphics[width=\textwidth]{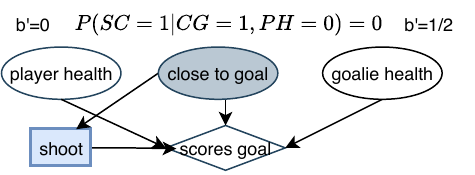}
         \caption{Prediction}
         \label{fig:counterfactual-confound}
     \end{subfigure}
     \caption{Evaluating the what-if counterfactual query $P(\mathit{SC}=1|\mathit{CG}=1,\mathit{SH}=0,\ado(\mathit{SH})=1))$ for the confounded offline model of~\Cref{fig:confound-scm}. Numbers indicate posterior probabilities of latent source variables given the query observations. 
     ~\Cref{fig:observe-confound}, Abduction: The posterior probability of the player being healthy is 0, given that they did not shoot.
     ~\Cref{fig:intervene-confound}, Intervention: The truncated model removes the link between Player Health and shooting and uses the posterior distribution over source variables. 
     ~\Cref{fig:counterfactual-confound}, Prediction: In the truncated model, the scoring probability is 0, given that we have inferred that the player is not healthy.
     \label{fig:compute-confound}}
      \end{minipage}


\begin{minipage}[t]{\textwidth} 
     \centering
     \begin{subfigure}[b]{0.32\textwidth}
         \centering
\includegraphics[width=\textwidth]{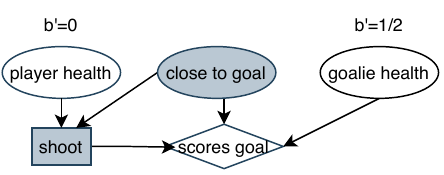}
         \caption{Abduction}
         \label{fig:observe-noconfound}
     \end{subfigure}
      \begin{subfigure}[b]{0.32\textwidth}
         \centering
         \includegraphics[width=\textwidth]{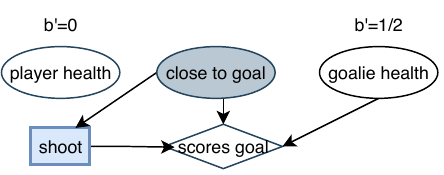}
         \caption{Intervention}
         \label{fig:intervene-noconfound}
     \end{subfigure}
     \begin{subfigure}[b]{0.32\textwidth}
         \centering
\includegraphics[width=\textwidth]{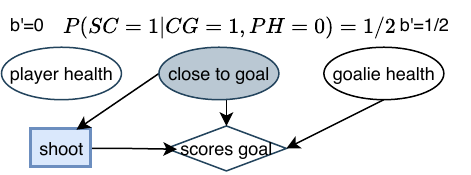}
         \caption{Prediction}
         \label{fig:counterfactual-non-confound}
     \end{subfigure}
     \caption{Evaluating the what-if counterfactual query $P(\mathit{SC}=1|\mathit{CG}=1,\mathit{SH}=0,\ado(\mathit{SH})=1))$ for the unconfounded model of~\Cref{fig:no-confounder}. Numbers indicate posterior probabilities of latent source variables given the query observations. 
     ~\Cref{fig:observe-noconfound}, Abduction: The posterior probability of the player being healthy is 0, given that they did not shoot.
     ~\Cref{fig:intervene-noconfound}, Intervention: The truncated model removes the link between player health and shooting and uses the posterior distribution over source variables. 
     ~\Cref{fig:counterfactual-non-confound}, Prediction: The probability of scoring is independent of the player health, and thus equal to the probability of the goalie not being healthy, which is 1/2.
     \label{fig:compute-non-confound}}
    \end{minipage}
\end{figure}

\subsection{Action Sufficiency and What-if Counterfactuals} \label{sec:causal-suffice}

For example the online model~\Cref{fig:first-scm} is action sufficient, so~\Cref{lemma:scm-action-sufficient} entails that 

In this section we prove an analog of~\Cref{lemma:observe-condition} for SCMs: Under action sufficiency, what-if, interventional, and conditional probabilities are equivalent. 

The condition of action sufficiency introduced for causal Bayesian networks required that all causes of the agent's actions should be observed (\Cref{sec:bns}). This definition needs to be modified for structural causal models, because non-deterministic variables require latent parents to generate variance. For example, in the linear structural equation $\yv = a \xv + b + \varepsilon$, the noise term $\varepsilon$ is a latent variable. For an RL example, consider a probabilistic policy $\policy(\A|\sv)$ where $\sv$ is a completely observable state. In an SCM, such a policy is represented by a structural equation $\A = f(\sv,\uv)$ where $\uv$ is a latent cause that generates the distribution over actions given an observed state. The causal modelling literature therefore utilizes a more general causal sufficiency condition that allows a variable to have a latent cause, but not a {\em shared} latent cause~\citep{Spirtes2000}. The insight is that what matters for causal modelling is not the presence of latent causes, but of latent {\em confounders}. Formally we say that source variable $\uv$ is a noise variable for $\xv$ if $\xv$ is the only child of $\uv$. A set of observed variables $\set{\ov}$ is \defterm{causally sufficient} for variable $\xv$ in a causal SCM graph $\G$ if every latent parent of $\xv$ is a noise variable for $\xv$. 
An SCM $\cm$ is \defterm{action sufficient} if its graph is causally sufficient for the action variable $\A$.\footnote{Our concept of action-sufficiency differs from the notion introduced by~\cite{huang2022action}. For them, a latent state is action sufficient if it is powerful enough to support an optimal policy based on the latent state space. Our concept could be called ``causal action sufficiency'' to disambiguate.} The graph of~\Cref{fig:confound-scm} is not action sufficient because Player Health is a latent common cause of Shooting and Scoring. The graph of~\Cref{fig:no-confounder} is action sufficient because the latent cause Player Health of the shooting action is not a cause of any other variable, hence a noise variable for shooting.
 
The next lemma states that under action sufficiency, what-if counterfactuals reduce to conditional probabilities.

\begin{lemma} \label{lemma:parent-condition-scm}
    Let $\cm$ be a probabilistic SCM and let $\set{\yv},\A,\set{\xv}$ be a disjoint set of random variables such that $\set{\xv}$ includes all parents of $\A$ except for possibly a noise variable $\uv_{\A}$ of $\A$ (i.e., $\set{\xv} \supseteq \PA_{\A} - \uv_{\A}$), and none of the descendants of $\A$. Then for any actions  $\a,\hat{\a}$ we have
    \[
    \mprob{\cm}(\set{\yv}|\set{\xv}=\set{\xvalue},\A=\a,\ado(\A = \hat{\a})) = \mprob{\cm}(\set{\yv}|\set{\xv}=\set{\xvalue},\ado(\A = \hat{\a})) = \mprob{\cm}(\set{\yv}|\set{\xv}=\set{\xvalue},\A = \hat{\a}).
    \]
\end{lemma}

~\Cref{lemma:parent-condition-scm} implies the next corollary, which states that under action sufficiency, a what-if counterfactual ``what would happen if I selected action $\hat{a}$ instead of action $a$'' can be evaluated by the conditional probability given that $\hat{a}$ is observed, without taking into account the actual action choice $a$.

\begin{lemma} \label{lemma:scm-action-sufficient}
    Let $\set{\ov} \subseteq \set{\rv}$ be an action sufficient set of observable variables in a probabilistic SCM $\cm$ that contains no effects (descendants) of the action variable $\A$. Then 
    \[
    \mprob{\cm}(\set{\yv}|\set{\ov}=\set{\ovalue},\A=\a,\ado(\A = \hat{\a})) = \mprob{\cm}(\set{\yv}|\set{\ov}=\set{\ovalue},\ado(\A = \hat{\a})) = \mprob{\cm}(\set{\yv}|\set{\ov}=\set{\ovalue},\A = \hat{\a}).
    \]
    for any actions $\a,\hat{\a}$ and any list of target outcomes $\set{\yv}$. 
\end{lemma} 

In a DDN that represents an MDP, the only effect of the action is the reward variable (\Cref{def:ddn}(3)). Thus not conditioning on an effect is equivalent to not conditioning on an observed reward, i.e. posing a what-if query. Using an observed reward outcome to predict the result of a counterfactual action requires a hindsight counterfactual, which we discuss in the next section.

 \section{Counterfactuals and Online Policy Evaluation} \label{sec:counterfactuals}

This section examines policy evaluation based on counterfactuals. 
We first consider what-if counterfactuals, then hindsight counterfactuals.

\subsection{What-if Counterfactuals}

In RL based on what-if counterfactual decisions, the learning agent 
seeks to learn an optimal policy that may deviate from the actual decisions taken. Following Pearl's suggestion in a similar context~\cite[Ch.4.1.1]{Pearl2000}, we refer to observed decisions by the behavioral agent as ``acts'', with associated random variable $\A$ and active decisions by the learning agent as ``actions'', with associated random variable $\hat{\A}$. 
For what-if counterfactuals, as with interventional probabilities, we show that under action sufficiency counterfactual rewards, transitions, and value functions are equivalent to observational rewards, transitions, and value functions. As with causal Bayesian networks, probabilistic SCMs can be straightforwardly extended to a dynamic causal model by specifying them with respect to both current variables $\set{\sv}, \A,\R$, and successor variables $\set{\sv}', \A',\R'$. For the remainder of the paper, we use the $\ddn$ notation to refer to a \defterm{dynamic SCM}. 

\begin{proposition} \label{prop:whatif-reward} Suppose that an observation signal $\set{\ov}$ is action sufficient in a dynamic probabilistic SCM $\ddn$. 
Then $\mprob{\ddn}(\R|\set{\ov},\hat{\A}) = \mprob{\ddn}(\R|\set{\ov},\ado(\hat{\A})) = \mprob{\ddn}(\R|\set{\ov},\A,\ado(\hat{\A}))$ and $\mprob{\ddn}(\set{\sv}'|\set{\ov},\hat{\A}) = \mprob{\ddn}(\set{\sv}'|\set{\ov},\ado(\hat{\A})) = \mprob{\ddn}(\set{\sv}'|\set{\ov},\A,\ado(\hat{\A}))$. \end{proposition}
\begin{proof}
    Follows immediately from~\Cref{lemma:scm-action-sufficient}. 
\end{proof}

The import of the proposition is that under action sufficiency, observing the acts of an agent does not impact reward and state transition probabilities. 
~\Cref{tab:scm-queries} illustrates how conditional and interventional reward probabilities are the same for the action sufficient model of~\Cref{fig:no-confounder} and different from the confounded model of~\Cref{fig:confound-scm}. 

The next proposition asserts that for action sufficient models, both conditional and what-if value and action value functions are the same. We model what-if counterfactuals by including in the learning agent's observation signal the acts of the behavioral agent. In symbols, we have $\set{\ov}^{\policy} = \set{\ov} \cup \A$ where $\set{\ov}$ is the set of state variables observable by the learning agent. 
The \defterm{Bellman equation for the counterfactual Q-function} is obtained by replacing the original observation signal with the expanded observation signal, either in the conditional Bellman~\Cref{eq:observe-ddn}, or in the interventional Bellman~\Cref{eq:intervene-ddn}. We write out the interventional counterfactual Bellman equation because it is the most complex. 

\begin{align}
    \Q^{\policy,\ddn}(\langle \set{\ovalue},\a,b \rangle,\ado(\hat{a})) =  R^{\ddn}(\langle \set{\ovalue},\a,b \rangle,\ado(\hat{a})) + \gamma \sum_{\set{\ovalue}'} \mprob{\ddn}(\set{\ovalue}'|\langle \set{\ovalue},\a,b \rangle,\ado(\hat{a})) V^{\policy,\ddn}(\set{\ovalue}',\a,b') \label{eq:w=hatif-ddn} \\
    \V^{\policy,\ddn}(\langle \set{\ovalue},\a,b \rangle) = \sum_{a} \policy(a|\langle \set{\ovalue},\a,b \rangle) \Q^{\policy,\ddn}(\langle \set{\ovalue},\a,b \rangle,\ado(\hat{a})) \nonumber\\
     R^{\ddn}(\langle \set{\ovalue},\a,b \rangle,\ado(\hat{a})) = E_{\set{\lvalue}\sim b(\set{\lvalue}|\ado(\hat{a}))} \sum_{r} \r \cdot \mprob{\ddn}(\R = \r|\set{\lvalue},\set{\ovalue},\a, \ado(\hat{a})) \nonumber
\end{align}

\begin{proposition} \label{prop:whatif-value} Suppose that an observation signal $\set{\ov}$ is action sufficient in a dynamic probabilistic SCM $\ddn$. The counterfactual $Q$-value equals the interventional and conditional $Q$-values for every epistemic state and action:
\begin{equation*}
\Q^{\ddn}(\langle \set{\ovalue},\a,b \rangle,\ado(\hat{a})) =  \Q^{\ddn}(\langle \set{\ovalue},b \rangle,\ado(\hat{a})) = \Q^{\ddn}(\langle \set{\ovalue},b \rangle,\hat{a}). 
\end{equation*}
\end{proposition}

The proof proceeds as in~\Cref{prop:online-value}: By~\Cref{prop:whatif-reward}, conditional and interventional probabilities agree on rewards and transitions by~\Cref{prop:online-reward}.  The resulting value functions then agree as well because they are defined recursively by reward and transition probabilities. 

\paragraph{Example}~\Cref{fig:whatif-diagram} shows the computation of values for the marginal policy $\marginal$, given that we observe a player not taking a shot when they are close to the goal. The confounded model of~\Cref{fig:confound-scm} implies that the player is not healthy. Since they score only if they are healthy, it follows that their expected reward is 0 regardless of what action we direct them to perform. In contrast, from~\Cref{fig:intervene-diagram-offline} we have that conditional only on being close to the goal, the scoring chance of the marginal policy is 1/6. 

The action sufficient model of~\Cref{fig:no-confounder} also implies that the player is not healthy from the same observations. However, in this model player health does not affect scoring, which depends only on goalie health. This means that conditional and interventional Q-values are the same, given closeness to goal and an agent's actions, as entailed by~\Cref{prop:whatif-value}. 


\begin{figure}
    \centering
    \includegraphics{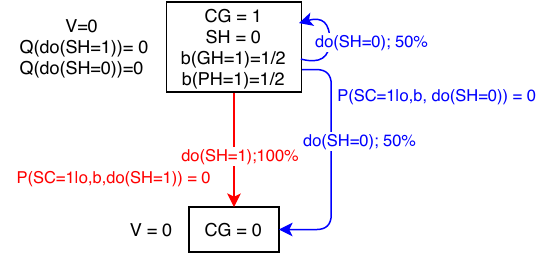}
    \caption{Off-policy policy evaluation based on what-if counterfactuals in the confounded offline model of~\Cref{fig:confound-scm}. The policy evaluated is the marginal behavioral policy of~\Cref{eq:marginal}, which chooses to shoot with probability 1/2 if the player is close to the goal. The evaluation uses interventional {\em what-if} reward and transition probabilities derived from the DDN of~\Cref{sec:ddn-example}. The diagram shows the V value and Q action values for the epistemic state where the agent is observed to be close to the goal, their belief is uniform over the latent variables, and {\em they do not take a shot.}}
    \label{fig:whatif-diagram}
\end{figure}

\subsection{Hindsight Counterfactuals and Online Policy Evaluation.} \label{sec:hindsight} A remarkable feature of counterfactual hindsight probabilities is that they can differ from conditional probabilities even in action sufficient settings, such as online learning. This subsection gives examples to illustrate the phenomenon, two examples for the reward model and one for policy evaluation. The general insight is that while it has long been noted in RL that {\em past} observations allow us to infer latent state information~\citep{bib:pomdp,Hausknecht2015}, {\em future} information allow us to infer current latent state information as well through hindsight. 

Suppose that we observe that in state $\set{\svalue}$, an act $\a$ was followed by a reward $\r$. We can then ask ``what would the reward have been if the agent had chosen the action $\ado(\hat{a})$'' instead? The corresponding \defterm{hindsight reward probability} is given by counterfactual queries of the form $P(\R_{\ado(\hat{a})}|\set{\ov},\A,\R)$.


A simple reward hindsight query in our sports example would be $$\mprob{\ddn}(\SC_{\SH=1} =1|\CG=1, \PH = 1, \SH = 1, \SC = 1).$$

To evaluate the hindsight reward probability in the online model of~\Cref{sec:ddn-example}, first we update the initial belief given the observations:

$$b(\GH=0|\CG=1, \PH = 1, \SC = 1) = 1.$$

Informally, since the player scores only if the goalie is not healthy, we can infer from their scoring that the goalie is not healthy. Given that the player is healthy and the goalie is not, the player is certain to score, so the hindsight reward probability is 1:

$$\mprob{\ddn}(\SC_{\SH=1} =1|\CG=1, \PH = 1, \SC = 1) = P(\SC=1|\CG=1, \PH = 1,\GH = 0, \ado(\SH=1)) = 1$$

Without hindsight, the chance of scoring is only 1/2, since the goalie has a 50\% chance of being healthy. A more interesting example is to consider not only immediate rewards, but hindsight based on future rewards (as in hindsight credit assignment~\citep{harutyunyan2019hindsight}). Suppose that a player does not take a shot, and then their team scores at the next time instant. Since this implies that the goalie is not healthy, we can infer that they would have scored if they had taken a shot earlier. Using counterfactual notation, we have 

$$\mprob{\ddn}(\SC_{\SH=1} =1|\CG=1, \PH = 1, \SH = 0, \SC' = 1) = P(\SC=1|\CG=1, \PH = 1,\GH = 0, \ado(\SH=1)) = 1$$

Hindsight Q-values can be defined by including the observed outcomes as part of the agent's observation signal. ~\Cref{fig:hindsight} shows how Q-values change with hindsight. If the player is made not to shoot, they have a 50\% chance of maintaining possession. If they maintain possession, they will shoot at the next step according to the behavioral policy, so they are certain to score then because the goalie is not  healthy. Hence their expected return after not shooting is 1/2; in symbols $$\Q^{\ddn}(\langle \CG=1, \PH = 1, \SC = 1, b\rangle, \ado(\SH=0)) = 1/2.$$ These examples illustrate how observing outcomes can be a powerful source of information about the latent environment state (such as goalie health). 

\begin{figure}
    \centering
    \includegraphics{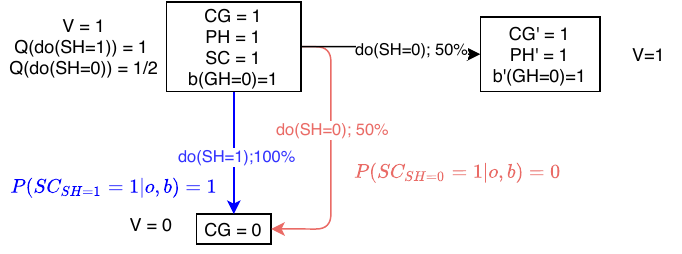}
    \caption{On-policy policy evaluation in the online setting 
    with hindsight. The policy evaluated is the standard behavioral policy. The evaluation uses interventional reward and transition probabilities derived from the online DDN of of~\Cref{sec:ddn-example}. The diagram shows the V value and Q action values for the epistemic state where the agent is observed to be close to the goal, their belief is uniform over the latent variables, and {\em the observation signal includes the current reward (goal scored)}.}
    \label{fig:hindsight}
\end{figure}

\section{Related Work: Current and New Research Directions} \label{sec:related}

We review related recent work and describe direction for future research. Our goal is not a comprehensive survey of causal RL, but to describe current research from the perspective of the distinction between observational, interventional, and counterfactual probabilities. As we have explained in this paper, this distinction is related to the distinctions between online vs. offline RL, and between the observation spaces of behavioral and learned policies. For surveys of causal RL, please
see~\cite{bib:barenboim-tutorial}, \citep{deng2023causal} and ~\citep[Section E]{scholkopf2021towards}. We organize our discussion in different sections corresponding to online RL, offline RL, and approaches that combine both online and offline learning. We focus on the causal RL tasks listed in the previous surveys. 

A common goal in previous research is leveraging a given causal model. Such approaches can be categorized as {\em causal model-based RL}. Causal model-based RL inherits the challenges and benefits of model-based RL in general~\citep[Sec.5]{Levine2020}. Our discussion focuses on the following special features of causal model-based RL:

\begin{enumerate}
    \item The ability to correctly evaluate the effects of interventions when confounders are present.
    \item The greater expressive power of causal models, which define not only conditional reward-transition models, but also interventional and counterfactual reward-transition models. 
    \item The causal graph structure, which decomposes joint distributions into local mechanisms. 
\end{enumerate}

\subsection{Online Causal RL}

Our analysis implies that in online learning, we can expect conditional probabilities to be unconfounded, which reduces the importance of advantage 1. However, the last two advantages (graphical structure and greater expressive power) apply to online learning as well, and have been leveraged in previous work on online RL.

\paragraph{Eliminating Irrelevant Variables: Model-based State Abstraction.} \label{sec:abstract}
~\Cref{lemma:parent-condition} implies that conditioning on a superset $\set{\xv}\supseteq \pa_{\A}$ of observable causes (parents) of $\A$ suffices to ensure that conditional probabilities are causal. 
This means that the conditioning set need not include {\em all} observed variables, that is, it need not include the {\em entire} state observed by the agent. For example in the car driving model of~\Cref{fig:car-online}, for the braking decision, it suffices to condition on $\mathit{Front Taillight}$ and $\mathit{Car in Front}$; the $\mathit{Own Brake Light Variable}$ can be ignored during the decision process. 
In general, a state variable $\S$ is \defterm{conditionally irrelevant} to a decision if $\S \perp \R, \set{\S'} | \A, \PA_{\A}$. Thus a causal model supports variable selection as a form of state abstraction/simplification ~\cite[Sec.8.2.2]
{peters2017elements}. The paper by ~\cite{sen2017identifying} is one of the first to leverage a given causal graph to reduce the effective state space. They prove that this reduction improves regret bounds in online bandit problems, when the bound is a function of the size of the state space. \cite{zhang2020designing} provide an algorithm for reducing the state space by eliminating irrelevant variables given a causal model, which leads to substantive improvement in regret bounds.~\cite{ICML22-wang,wang2024building} relate the elimination of irrelevant variables from a dynamic causal model to state abstraction. 

\paragraph{Data Augmentation and Hindsight Counterfactuals.}  One of the traditional uses of models in RL, going back to the classic Dyna system~\citep{sutton1990integrated}, is to augment the observed transition data with virtual experiences simulated from the model. \cite{bib:hindsight-augment} utilize hindsight counterfactuals to generate virtual state transitions that specify the next state that would have occurred in the same scenario.
These counterfactual state transitions take the form $P(\set{S}^{'}_{\hat{\A}}|\set{S},\set{S}',\A)$, where we observe a next state transition from $\set{S}$ to $\set{S}'$ and ask what the next state would have been if the agent had selected action $\hat{A}$ instead of $\A$. 

Compared to traditional state-transition models of the form $P(\set{S'}|\set{S},\hat{A})$, hindsight counterfactuals condition on more information and thus are potentially more accurate in generating virtual transitions. \citet{bib:hindsight-augment} provide empirical evidence that hindsight state transitions speed up learning an optimal policy. Generating hindsight state transitions requires a causal model and is not possible with a traditional RL transition model that is based on conditional probabilities only. An open topic for future research is using hindsight counterfactuals to generate roll-outs that are longer than single transitions, as is common in other model-based approaches~\citep{janner2019trust,sutton1990integrated}. 

Another open topic for future research is leveraging for data augmentation hindsight {\em rewards}, which take the form $P(\set{R}^{'}_{\hat{\A}}|\set{S},\set{R}',\A)$.  A data augmentation example using hindsight reward counterfactuals would be as follows: Having observed the transition $$(CG_{t}=1, PH_{t} = 1, SH_{t}=0; CG_{t+1}=1, PH_{t+1} = 1, SH_{t+1}=1, SC_{t+1} = 1),$$ 
we can augment the data with the counterfactual outcome 
$$(CG_{t}=1, PH_{t} = 1, SH_{t}=1, SC_{t} = 1)$$
following the reasoning of~\Cref{sec:hindsight}. 
Hindsight rewards are considered in the well-known hindsight experience replay approach~\citep{bib:hindsight}. Hindsight experience replay is based on evaluating multiple goals, corresponding to different reward signals. The dynamic causal models considered in this paper assume a single reward function. Dynamic causal models offer a promising approach to {\em hindsight credit assignment} \citep{harutyunyan2019hindsight}.



\paragraph{Learning A Causal Model from Online Data.} Causal model discovery methods that are applicable to online RL learning include deep models based on auto-encoders~\citep{lu2018deconfounding} and GANs~\citep{bib:hindsight-augment}. Such deep generative models generate observations $\set{x}$ from latent variables $\set{z}$ (cf.~\Cref{eq:evaluate}), but they are not {\em structural} causal models based on a causal graph that represents local causal mechanisms.~\citep{huang2022action} show how constraints from a given causal graph can be leveraged to learn latent state representations. Learning an influence diagram over state variables from online RL data seems to be a new research topic. An exciting new possibility for online learning is that the agent's exploration can include experimentation in order to ascertain the causal structure among the state variables.

\paragraph{Counterfactual Regret.}
\cite{bib:barenboim-tutorial} describes an online counterfactual regret optimization procedure where the agent conditions on their intending to perform action $A$ {\em before} they actually execute action $A$. For example, their policy may recommend $A = \policy(\set{\svalue})$, the agent conditions on this information, but considers alternative actions $A'$. If the policy recommendation $A = \policy(\set{\svalue})$ carries information about the agent's internal state.
\citeauthor{bib:barenboim-tutorial} gives the example of a gambler in a casino whose choices are influenced by how drunk they are; therefore intending a risky gamble should give them pause to reconsider. While it is possible for human agents to be unaware of what causes their intentions, for an agent to be opaque to themselves in this way raises deep philosophical and psychological questions about free will and intentionality, as \citeauthor{bib:barenboim-tutorial} notes. In the case of a {\em computational agent} implemented by a computer program, the program's interface defines the possible inputs and hence the causes of its outputs. Thus when the behavioral agent is implemented by a computer program to which the learning agent has access, as they do in online learning, the causes of the behavioral agent's decisions are unlikely to be opaque to the learning agent. 


\subsection{Offline Causal RL} We discuss causal offline policy evaluation and imitation learning. 

\subsubsection{Offline Off-policy Evaluation} \label{sec:ope} The off-policy evaluation (OPE) problem is to estimate the value function of a learned policy that is different from the behavioral policy. OPE is one of the major approaches to offline RL, where information about the behavioral policy is recorded in a previously collected dataset~\citep{Levine2020}. 
The examples in this paper illustrated OPE based on a causal model. We selectively discuss OPE work related to causality and causal modelling. 

\paragraph{Offline RL and Distribution Shift.} As we mentioned in the introduction, \cite{Levine2020} assert that ``offline reinforcement learning is about making and answering counterfactual queries.'' The key counterfactual  for offline RL  is about what might happen if an agent followed a policy different from the one observed in the data. In a car driving scenario (see~\Cref{fig:driving}), we can ask based on expert driver data, what might happen if a beginner drives the car. As our examples show, the causal counterfactual semantics presented in~\Cref{sec:counterfactuals} works well when the causal model supports inferring latent features of the same agent/environment, which facilitates counterfactual predictions. How causal models can be leveraged to answer counterfactual questions that involve changing agents is a valuable research topic for causal RL. It may be possible to address this in the transportability framework, which studies the extent to which causal mechanisms valid in one environment can be applied in (``transported to'') another~\cite{bareinboim2014transportability,correa2022counterfactual,bib:barenboim-tutorial}. 

Another approach is to view offline RL as an instance of {\em distribution shift}:~\citep{Levine2020}: while the dataset distribution over trajectories is based on the behavioral policy, the learned policy needs to be evaluated on the distribution reflecting its choices. Distribution shift departs from the i.i.d. assumption made in much machine learning work, because the training distribution differs from the test distribution~\citep[Sec.7]{Levine2020}. If causal mechanisms are shared between training and test distribution, as is often the case due to their local scope, they constrain the extent of distribution shift~\citep{scholkopf2021towards}. To illustrate this point, consider how can we leverage the causal model of~\Cref{fig:first} to evaluate the marginal policy $\marginal$ that directs a player to shoot with probability 1/2 if they are close to the goal; see~\Cref{fig:marginal-online}. Changing the policy affects only the parents of the shooting action variable. In causal terminology, this means that {\em the causal mechanisms governing the reward variable Scores Goal is the same even if the policy changes.} In MDP terminology, the reward model is invariant; the same reward model can be used to evaluate both the behavioral and the new policy. Similarly, the next state transitions (not shown) are invariant under a policy change. 

The difficulty in modelling distribution shift arises from the presence of unobserved confounders. In the offline setting, Player Health is an unobserved confounder. If we use the same causal relations over the observable variables, as shown in~\Cref{fig:marginal-offline}, the Bayesian network requires estimating conditional reward probabilities of the form $P(\SC|\SH,\CG)$ to evaluate the new policy. However, as we saw in~\Cref{sec:cbn-example}, without observing Player Health, such conditional probabilities are confounded and do not correctly estimate interventional probabilities. On-going research in causal OPE addresses this issue.

\begin{figure}
     \centering
     \begin{subfigure}[b]{0.40\textwidth}
         \centering
\includegraphics[width=\textwidth]{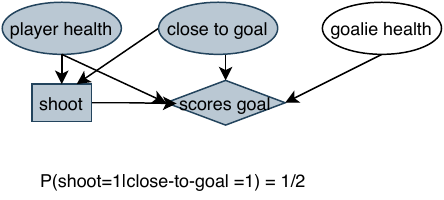}
         \caption{A new policy in the online model of~\Cref{fig:first}}
         \label{fig:marginal-online}
     \end{subfigure}
      \begin{subfigure}[b]{0.40\textwidth}
         \centering
         \includegraphics[width=\textwidth]{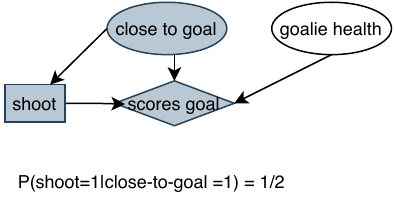}
         \caption{A new policy in the offline model of~\Cref{fig:confound}}
         \label{fig:marginal-offline}
     \end{subfigure}
     \caption{Leveraging a causal model to evaluate a policy different from the behavioral policy. The evaluation policy selects shooting with probability 1/2 if the player is close to the goal, and does not shoot if the player is far from the goal.
     ~\Cref{fig:marginal-online}: A causal graph for the new policy in the online setting. Player Health no longer is a cause of shooting. The graph remains the same for all other variables.
     ~\Cref{fig:marginal-offline}: A causal graph for the new policy in the offline setting, where Player Health is unobserved. The graph agrees with the online graph for all observable variables. Note that the conditional probability $P(\SC|\SH,\CG)$ is confounded because Player Health is not observed.
     \label{fig:marginal}}
\end{figure}



\paragraph{Causal OPE.} Most work is based on confounded MDPs~\citep{zhang2016markov,kausik2024offline,bruns2021model}, which are similar to the POMDP model of~\Cref{sec:pomdps}: States are decomposed into an observed part and an unobserved part. The behavioral policy is assumed to depend on the complete state, whereas the policy to be evaluated depends on the observation signal only. A current line of research gives bounds on the extent of the bias due to spurious correlations, based on assumptions about the confounders~\citep{kausik2024offline,bruns2021model}. For example, ``memoryless confounders'' are sampled independently at each time instant. At the other extreme, time-invariant latent variables such as Goalie Health or Player Health in our example are known as ``global confounders''~\citep{kausik2024offline}. Some of the causal OPE methods are model-based in the sense of estimating transition probabilities, but not in the sense of utilizing a dynamic SCM. 
Another difference is that the evaluation policies considered are Markovian in that they depend on current observations only (cf.~\Cref{sec:ddn-example}). In contrast, our belief MDP framework allows for evaluation policies that depend on past observations/current beliefs.

\paragraph{Leveraging a Causal Model for OPE} ~\cite{wang2021provably} propose policy optimization based on interventional probabilities, such as $P(\SC|\ado(\SH),\CG)$ rather than conditional probabilities, such as $P(\SC|\SH,\CG)$. To estimate the interventional probabilities, they assume that a causal model over the entire state space is available. The problem is then to compute a marginal interventional probability such as $P(\SC|\ado(\SH),\CG)$ from conditional probabilities over the entire state space, such as $P(\SC|\SH,\CG,\PH)$. Pearl's do-calculus provides powerful techniques for inferring marginal interventional probabilities from conditional probabilities, using what are known as {\em adjustment formulas}. 
Two well-known types of adjustment formula are the backdoor and the frontdoor criterion.~\cite{wang2021provably} utilize both to compute causal reward and state-transition probabilities from a given causal model, and show how to use the causal probabilities in an interventional Bellman equation. To illustrate the idea in our sports example, consider offline learning in the confounded model of~\Cref{fig:confound}. Since in this setting, only closeness-to-goal is observable, an executable policy would be based on this variable only (e.g. ``shoot whenever you are close to the goal''). Finding interventional values for such a policy involves computing interventional probabilities such as $P(\SC=1|\CG=1, \ado(\SH=1))$. According to the backdoor criterion, such probabilities can be computed by marginalizing over the unobserved values of player health as follows: 

\begin{align*}
    P(\SC=1|\CG=1, \ado(\SH=1)) = \\
    P(\PH = 1) P(\SC=1|\PH = 1, \CG=1, \SH=1)) +
    P(\PH = 0) P(\SC=1|\PH = 0, \CG=1, \SH=1)) \\
    = 1/2 \cdot 1/2 + 1/2 \cdot 0  = 1/4
\end{align*}

which agrees with the result of~\Cref{tab:probs}. While adjustment formulas provide an elegant approach to addressing spurious correlations in OPE, it is not entirely clear what their use case is for offline learning, as the learning agent does not have access to the latent variables that appear in the adjustment formulas. A possibility is that the behavioral agent uses their access to latent variables (e.g., the athlete has access to their health) to compute the marginal interventional probability and communicate it to the offline learner. 

\paragraph{Learning a Causal Model from Offline Data.} Model-based offline RL is a major approach to offline RL \cite[Sec.5.2]{Levine2020}. Many of the approaches designed for model-based offline RL can also be applied for offline RL based on causal models. 

A specific possibility for causal models is to learn a fully specified latent variable model, such as the model of~\Cref{fig:confound}, and use the backdoor adjustment with latent variables to deconfound interventional probabilities~\citep{lu2018deconfounding}.
For learning a causal model offline, it is likely that the extensive work on learning Bayesian networks for temporal data can be leveraged, including recent approaches based on deep learning and continuous optimization~\citep{ntsnotearssun2023}. Latent variable models related to deep generative models are a promising direction for learning structural causal models~\citep{geffner2022deep,Sun2023,mooij2016distinguishing,hyvarinen2010estimation}. Much of the research on learning a causal model assumes causal sufficiency (no confounders), but not all of it~\citep{scholkopf2021towards}. Learning an influence diagram with a causal graph over state variables from offline RL data seems to be a new research topic. ~\cite{pearl2018theoretical} emphasizes that domain knowledge is important in constructing a causal graph precisely because causal probabilities go beyond directly observable correlations. Most causal graph learning packages allow users to specify background knowledge causal connections that should be present or absent~\citep{Spirtes2000}.

\subsubsection{Imitation Learning and Behavioral Cloning.} \label{sec:cloning}

Imitation learning is a major approach to RL when the agent does not have the ability to interact directly with the environment online. Rather than learning an optimal policy from offline data, the goal is to learn a policy that matches an expert demonstrator.~\cite{Zhang2020} introduce the concept of different observation spaces for the imitator and demonstrator, and state necessary and sufficient graphical conditions on a given causal domain model for when an imitating policy can match the expected return of a demonstrator, even when the observation signals of the imitator and demonstrator are different. They show that under observation-equivalence, when imitator and demonstrator share the same observation signal (\Cref{sec:obs-equiv}), imitation is always possible~\cite[Thm.1]{Zhang2020}. This theory is the closest previous work to our results on observation-equivalence. 

The notion that imitation consists in matching the demonstrator's return is not the standard concept of imitation learning in RL; for example if the observation spaces are different, matching returns may require the imitator to follow a policy that is quite different from that of the demonstrator~\cite[Fig.3b]{Zhang2020}. A more usual concept of imitation in RL, known as {\em behavioral cloning}, is that the learned policy should be similar to that of the behavior policy that generated the data. Using our notation, the goal is to learn a policy $\policy$ such that 
\begin{equation} \label{eq:bc-observe}
    b(\A|\set{S}) \approx \policy(\A|\set{S})
\end{equation}

where $b$ is the behavior policy to be imitated.~\cite{de2019causal} in a paper on ``causal confusion'' point out that the conditional probability behavioral cloning objective~\cref{eq:bc-observe} can be problematic, even assuming observation-equivalence. For example, in a self-driving car example like that of~\Cref{fig:driving}, the observation signal includes the driver's own brake light. Since the brake light comes on only when the driver brakes, we have a very high correlation between braking and the brake light being on:

$$P(\mathit{Brake}=1|\mathit{Brake Light} = 1) \approx 1 \mbox{ and } P(\mathit{Brake}=1|\mathit{Brake Light} = 0) \approx 0.$$ 

Therefore if the imitator matches conditional probabilities according to the criterion~\cref{eq:bc-observe}, they will brake only if they observe the brake light coming on. 
In this scenario, the imitator will not respond to the location of the other cars and their tail lights, which fails to match the expert driver's behavior (and fails to avoid accidents). 

A causal diagnosis of the problem is that the imitator confused causes of acts with correlates of acts (cf.~\cite[Ch.4.1.1]{Pearl2000}). 
As we discussed in~\Cref{sec:intro}, a
high conditional probability such as $P(\mathit{Brake}=1|\mathit{Brake Light} = 1)$ is not causal if we condition on an effect (light) to predict a cause (brake). One remedy is to remove from the imitator's observation space effects of their actions, which also has the benefit of reducing the effective state space (see~\Cref{sec:abstract} above). A more fundamental approach is to redefine the behavioral cloning objective in terms of causal effects:

\begin{equation} \label{eq:bc-intervene}
    b(\A|\ado(\set{S})) \approx \policy(\A|\ado(\set{S}))
\end{equation}

This \defterm{causal behavioral cloning} objective can be interpreted as the query ``what would the agent do if we were to put them into a state $\set{S}$''? The causal objective eliminates the potential confusion due to non-causal correlations. To illustrate this point in our running example, in~\Cref{fig:driving} we have

$$P(\mathit{Brake}=1|\ado(\mathit{Brake Light} = 1)) = P(\mathit{Brake}=1) << 1.$$

Causal behavioral cloning is a new topic for future research in causal RL.

\subsection{Offline + online RL} 

In this hybrid setting, the learning agent interacts with the environment to collect more data, but a prior dataset is also utilized~\citep{sutton1990integrated,janner2019trust}. We can view this as an off-policy setting (where all available data from any policy are used; see~\Cref{fig:levine-offline}). A successful example is the AlphaGo system, which used an offline dataset of master games to find a good initial policy through imitation learning, then fine-tuned the policy with self-play~\citep{silver2016mastering}. 
An active topic of research is how causal models can leverage offline datasets for online learning~\citep{gasse2021causal,zhang2020designing}, sometimes called generalized policy learning~\citep{bib:barenboim-tutorial}.

Off-policy model-based RL approaches~\cite[Sec.5.2]{Levine2020} can be applied with causal models to leverage the offline dataset. The techniques we outlined for causal models in the online and offline settings can be utilized in the hybrid offline/online setting as well. For example, we can use learn a causal graph from the offline dataset, and fine-tune it during online interactions with the environment. The ability to intervene in an environment in the online setting is potentially a powerful tool for learning a dynamic influence model. For instance, performing experiments can resolve the causal direction between two variables which may not be possible from observational data alone, at least not without assuming causal sufficiency. 

\cite{zhang2020designing} describe a hybrid approach for joint exploration and policy learning that is not based on a model. The offline dataset is used to estimate conditional state-reward transition probabilities. These estimates may not be correct interventional probabilities if confounders are present. Causal inference theory has established theoretical bounds on how far a conditional probability can differ from the interventional probability. \citeauthor{zhang2020designing} use the resulting interval estimates for interventional probabilities as an input to optimistic exploration for online learning, where policies are evaluated according to their maximum possible value. Optimism in the face of uncertainty is a well-known approach in RL for ensuring extensive exploration of the state space~\citep{osband2014near}. The offline dataset provides tighter bounds on interventional probabilities than learning from online data only, thereby speeding up optimistic exploration. 

\section{Conclusion} We believe that many RL researchers share the intuition that in common traditional RL settings, conditional probabilities correctly estimate the causal effects of actions. Our paper spelled out the conditions where we can expect conditional probabilities to correctly measure causal effects. We provided a rigorous argument, using the formal semantics of causal models, for why these conditions lead to correct estimates. Our conclusion is that it is only in offline 
off-policy learning with partially observable environments that conditional probabilities can diverge from observed conditional probabilities. The reason is that in this learning setting, the environment may contain unobserved confounders that influence both the decisions of the behavioral agent and the states and rewards that follow these decisions. Such confounders can introduce spurious correlations between decisions and states/rewards that do not correctly estimate the causal impact of the agent's decisions. 


In contrast, in an online or completely observable environment, such confounders are not present, and therefore conditional probabilities correctly reflect causal effects. Our argument for this conclusion involves two steps. (1) In such environments, the set of variables that the learning agent can observe is causally sufficient for the behavioral agent's actions, in the causal modelling sense that it includes all common causes of the behavioral agent's decisions and other variables. (2) We prove formally, using Pearl's do-calculus~\citeyearpar{Pearl2000}, that if a set of observable variables is causally sufficient for actions, then actions are not confounded with states or rewards. Thus our analysis relies on the important distinction between the observation signals available to the learning and behavioral agents, which has been highlighted by previous work in causal RL~\citep{Zhang2020,zhang2016markov,kausik2024offline}. 

In addition to the causal effects of interventions, causal models provide a rigorous specification of counterfactual probabilities through a fomral semantics. Causality researchers have recently proposed using counterfactuals to enhance reinforcement learning~\citep{bib:barenboim-tutorial,deng2023causal}. 
We therefore extend our analysis to counterfactuals, distinguishing two kinds of counterfactuals: what-if queries (e.g. if I choose action $a'$ in state $s$ instead of action $a$, what is the likely reward?) and hindsight counterfactuals that condition on an observed outcome (e.g. given that I received reward $r$ after choosing action $a$ in state $s$, what is the likely reward if I choose action $a'$ instead?). We show that in an online or completely observable environment, what-if queries can be correctly estimated from conditional probabilities, but hindsight counterfactuals go beyond conditional probabilities (cf.~\cite{bib:hindsight-augment}). 

Based on our analysis, we discussed the potential benefits of causal models in different reinforcement learning settings, such as online, offline, and off-policy. The most straightforward, though not the only, approach is to follow a model-based RL framework, replacing the traditional models involving conditional probabilities with a causal model. Structural causal models offer three main benefits: (1) They distinguish interventional and conditional probabilities, and therefore causation from correlation. (2) They factor the dynamics of a complex environment into local causal mechanisms represented in a causal graph. Local mechanisms are typically invariant under interventions~\citep{scholkopf2021towards}, which means that a causal graph can help address the challenge of distribution shift~\citep{Levine2020}: evaluating a learned policy against data gathered by another policy. (3) Causal models have greater expressive power than conditional probability models, since they can evaluate causal effects and counterfactual queries. We described existing work and promising future directions for how the benefits of causal models can be leveraged for reinforcement learning. 

Reinforcement learning and causality are areas of AI and machine learning that naturally complement each other. The analysis in this paper provides a guide for reinforcement learning researchers as to when and how they can make use of causal concepts and techniques to advance reinforcement learning. 

\subsubsection*{Acknowledgements} This research was supported by Discovery grants to each author from the Natural Sciences and Engineering Council of Canada. We are indebted to Mark Crowley and Ke Li for helpful discussion. 

\bibliography{master,new-refs}
\bibliographystyle{tmlr}

\appendix


\section{Proof of~\Cref{lemma:parent-condition,lemma:parent-condition-scm}}

\label{sec:observe-proof}

\setcounter{lemma}{0}

\begin{lemma} 
     Let $\cn$ be a causal Bayesian network and let $\set{\yv},A,\set{\xv}$ be a disjoint set of random variables such that $\set{\xv} \supseteq \PA_{\A}$. Then $\mprob{\cn}(\set{\yv}|\set{\xv}=\set{\xvalue},\ado(\A = \hat{\a})) = \mprob{\cn}(\set{\yv}|\set{\xv}=\set{\xvalue},\A = \hat{\a})$. 
\end{lemma}

\begin{proof}
Let $\set{V}$ denote the set of variables distinct from $A$. Consider disjoint variables $\set{X},\set{\yv},A$. Let $U = V - \set{X} \cup \set{\yv} \cup A$ be the set of remaining random variables. If the manipulated variable $A$ has no parents, the truncated and non-truncated model are the same, and the result follows immediately. Otherwise write $\pa(\set{v})$ for the assignment of values to the parents of $A$ defined by the values $\set{v}$. Similarly, we write $\pa(\set{x})$ to denote the assignment of values to the parents of $\A$ specified by the values $\set{X}= \set{x}$. Since $\set{X}$ includes all parents of $\A$, we have $\pa(\set{u},\set{\yvalue},\set{x}) = \pa(\set{x})$. 

The truncation semantics implies that for any assignment $\set{v},a$, we have 

\begin{equation*} \label{eq:factor-a}
  P(\set{V} = \set{v},\A = \hat{a})=  P_{\ado(A = \hat{a})}(\set{V} = \set{v},\A = \hat{a}) P(\hat{a}|\pa(\set{v})) 
\end{equation*}
That is, the joint distribution differs only by including a term for the conditional probability of $\A$ given its parents. Now  we have

\begin{align*}
    P(\set{\yvalue}|\set{x},\hat{a}) \\
    = \frac{\sum_{\set{u}} P(\set{u},\set{\yvalue},\set{x},\hat{a})}{\sum_{\set{u}',\set{\yvalue}'} P(\set{u}', \set{\yvalue}',\set{x},\hat{a})} \\
    = \frac{\sum_{\set{u}} P_{\ado(\hat{a})}(\set{u},\set{\yvalue},\set{x},\hat{a})
P(\hat{a}|\pa(\set{u},\set{\yvalue},\set{x}))}{\sum_{\set{u'},\set{\yvalue}'} P_{\ado(\hat{a})}(\set{u'}, \set{\yvalue}',\set{x},\hat{a})P(\hat{a}|\pa(\set{u'},\set{\yvalue}',\set{x}))} 
\\
= \frac{\sum_{\set{u}} P_{\ado(\hat{a})}(\set{u},\set{\yvalue},\set{x},\hat{a})
P(\hat{a}|\pa(\set{x}))}{\sum_{\set{u'},\set{\yvalue}'} P_{\ado(\hat{a})}(\set{u'}, \set{\yvalue}',\set{x},\hat{a})P(\hat{a}|\pa(\set{x}))} 
\\
= \frac{\sum_{\set{u}} P_{\ado(\hat{a})}(\set{u},\set{\yvalue},\set{x},\hat{a})}{\sum_{\set{u'},\set{\yvalue}'} P_{\ado(\hat{a})}(\set{u'}, \set{\yvalue}',\set{x},\hat{a})} \\
= P(\set{\yvalue}|\set{x},\ado(\hat{a}))
\end{align*}
\end{proof}

\setcounter{lemma}{2}

\begin{lemma} 
    Let $\cm$ be a probabilistic SCM and let $\set{\yv},\A,\set{\xv}$ be a disjoint set of random variables such that $\set{\xv}$ includes all parents of $\A$ except for possibly a noise variable $\uv_{\A}$ of $\A$ (i.e., $\set{\xv} \supseteq \PA_{\A} - \uv_{\A}$), and none of the descendants of $\A$. Then for any actions  $\a,\hat{\a}$ we have
    \[
    \mprob{\cm}(\set{\yv}|\set{\xv}=\set{\xvalue},\A=\a,\ado(\A = \hat{\a})) = \mprob{\cm}(\set{\yv}|\set{\xv}=\set{\xvalue},\ado(\A = \hat{\a})) = \mprob{\cm}(\set{\yv}|\set{\xv}=\set{\xvalue},\A = \hat{\a}).
    \]
\end{lemma}

\begin{proof}


The source variable posterior satisfies the following independence conditions:

\begin{align*}
    b(\set{\uvalue},\uvalue_{\A}|\set{\xv}=\set{\xvalue},\A=\a) =  b(\set{\uvalue}|\set{\xv}=\set{\xvalue},\A=\a) \cdot  b(\uvalue_{\A}|\set{\xv}=\set{\xvalue},\A=\a) \\
    b(\set{\uvalue}|\set{\xv}=\set{\xvalue},\A=\a) = b(\set{\uvalue}|\set{\xv}=\set{\xvalue}) \\
    b(\set{\uvalue},\uvalue_{\A}|\set{\xv}=\set{\xvalue}) =   b(\set{\uvalue}|\set{\xv}=\set{\xvalue}) \cdot   b(\uvalue_{\A}|\set{\xv}=\set{\xvalue})
\end{align*}

The first independence holds because $\set{\xv}$ contains $\A$, the only neighbor of $\uvalue_{\A}$, and all the parents of $\A$ (i.e., the entire Markov blanket of $\uvalue_{\A}$). The second independence holds because contains all the parents of $\set{\xv}$ and none of its descendants, so by the Markov condition, $\A$ is  independent of all its non-descendants. Since $\set{\uvalue}$ contains only source variables, it contains no descendent of $\A$. Similarly, the third independence holds because $\uvalue_{\A}$ is independent of all its non-descendants, and $\set{\xv}$ contains no descendent of $\A$ and hence no descendant of $\uvalue_{\A}$. 

Now consider the evaluation of the counterfactuals $\mprob{\cm}(\set{\yv}|\set{\xv}=\set{\xvalue},\A=\a,\ado(\A = \hat{\a}))$ and $\mprob{\cm}(\set{\yv}|\set{\xv}=\set{\xvalue},\ado(\A = \hat{\a}))$. Each probability is calculated in the same submodel $\scm_{\hat{a}}$ but with different posteriors. Let $\cm_1$ be the submodel with source posterior distribution $b(\set{\uvalue},\uvalue_{\A}|\set{\xv}=\set{\xvalue},\A=\a)$ and let $\cm_2$ be the submodel with source posterior distribution $b(\set{\uvalue},\uvalue_{\A}|\set{\xv}=\set{\xvalue})$. In each submodel, $\A$ is not generated by source variables but manipulated to the value $\hat{a}$. Let $\set{\uv}$ be the set of source variables other than $\uv_{\A}$. For an assignment of values to variables $\set{\wv} = \set{\wvalue}$, where $\A \not\in \set{\wv}$, let $\uv_{\set{\wvalue}|\a}$ be the set of assignments to the source variables $\set{\uv}$ such that the recursive solution procedure generates the assignment $\set{\wv} = \set{\wvalue}$ if variable $\A = \a$. Together with the independence conditions above, we therefore have the following:

\begin{align*}
    \mprob{\cm_1}(\set{\yv} = \set{\yvalue},\set{\xv}=\set{\xvalue})\\ =  \sum_{\set{\uvalue} \in \uv_{\set{\xvalue},\set{\yvalue}|\hat{\a}}} \sum_{\uvalue_{\a}}  b(\set{\uvalue},\uvalue_{\A}|\set{\xv}=\set{\xvalue},\A=\a) \\ 
    = \sum_{\set{\uvalue} \in \uv_{\set{\xvalue},\set{\yvalue}|\hat{\a}}} b(\set{\uvalue}|\set{\xv}=\set{\xvalue}) \sum_{\uvalue_{\a}} b(\uvalue_{\A}|\set{\xv}=\set{\xvalue},\A=\a) \\
    = \sum_{\set{\uvalue} \in \uv_{\set{\xvalue},\set{\yvalue}|\hat{\a}}} b(\set{\uvalue}|\set{\xv}=\set{\xvalue}) \cdot 1 \\
    = \sum_{\set{\uvalue} \in \uv_{\set{\xvalue},\set{\yvalue}|\hat{\a}}} b(\set{\uvalue}|\set{\xv}=\set{\xvalue}) \sum_{\uvalue_{\a}} b(\uvalue_{\A}|\set{\xv}=\set{\xvalue}) \\
    =  \mprob{\cm_2}(\set{\yv} = \set{\yvalue},\set{\xv}=\set{\xvalue})
\end{align*}

Since the joint probabilities are the same for each posterior, so are the conditional counterfactual probabilities, which establishes the first equality of the Lemma. 

The second equality follows as in~\Cref{lemma:parent-condition}: Given a parent assignment $\pa_{\A} = \set{\xvalue}$, we can define a conditional probability over actions by $P(\a|\set{\xvalue}) = \sum_{\uvalue_{\A}:f_{\A}(\set{\xvalue},\uvalue)=\a} b(\uvalue_{\A})$, that is, summing over the set of noise variable variables that generate the observed action. The conditional and interventional distributions differ only by this term, which does not depend on the target $\set{\yv}$ and therefore cancels out in the conditional probability. 

\end{proof}

\end{document}